\documentclass[12pt]{article}
\usepackage{fullpage}
\usepackage[round]{natbib}
\usepackage{float}
\usepackage{amsmath}
\usepackage{amssymb}
\usepackage{amsthm}
\newtheorem{theorem}{Theorem}
\newtheorem{lemma}{Lemma}
\newtheorem{proposition}{Proposition}
\newtheorem{definition}{Definition}
\usepackage{amsfonts}
\usepackage{times}
\usepackage{stmaryrd}
\usepackage{hyperref}
\usepackage{mwe}
\usepackage{graphicx} 
\usepackage{wrapfig} 
\usepackage{subcaption}
\usepackage{booktabs}
\usepackage{array}
\renewenvironment{proof}[1][Proof]{\noindent\textbf{#1.} }{\qed}

\usepackage[page]{appendix}
\usepackage{etoolbox}

\renewcommand{\appendixtocname}{Appendix Contents.}

\makeatletter
\renewcommand{\appendices}{%
  \clearpage
  \renewcommand{\thesection}{\Alph{section}}
  \renewcommand{\appendixname}{}
  \renewcommand{\appendixpagename}{}
  \let\tf@toc\tf@app
  \addtocontents{app}{\protect\setcounter{tocdepth}{2}}
  \immediate\write\@auxout{%
    \string\let\string\tf@toc\string\tf@app^^J
  }
}

\newcommand{\listofappendices}{%
  \begingroup
  \renewcommand{\contentsname}{\appendixtocname}
  \let\@oldstarttoc\@starttoc
  \def\@starttoc##1{\@oldstarttoc{app}}
  \tableofcontents
  \endgroup
}
\usepackage[utf8]{inputenc}
\newcommand{\BEAS}{\begin{eqnarray*}}
\newcommand{\EEAS}{\end{eqnarray*}}
\newcommand{\BEA}{\begin{eqnarray}}
\newcommand{\EEA}{\end{eqnarray}}
\newcommand{\BEQ}{\begin{equation}}
\newcommand{\EEQ}{\end{equation}}
\newcommand{\BIT}{\begin{itemize}}
\newcommand{\EIT}{\end{itemize}}
\newcommand{\BNUM}{\begin{enumerate}}
\newcommand{\ENUM}{\end{enumerate}}
\newcommand{\BA}{\begin{array}}
\newcommand{\EA}{\end{array}}

\newcommand{\Diag}{\mathop{\textnormal{ Diag}}}
\newcommand{\rb}{\mathbb{{R}}}

\def \ds { \displaystyle}

\def \E{{\mathbb E}}

\title{\textbf{An Uncertainty Principle\\ for Linear Recurrent Neural Networks}}
\author{
  \textbf{Alexandre Fran\c{c}ois}\\
  INRIA, Ecole Normale Sup\'erieure, PSL Research University, France\\
  \texttt{alexandre.francois@inria.fr}
  \and
  \textbf{Antonio Orvieto}\\
  MPI for Intelligent Systems, ELLIS Institute T\"ubingen, Germany\\
  \texttt{antonio@tue.ellis.eu}
  \and
  \textbf{Francis Bach}\\
  INRIA, Ecole Normale Sup\'erieure, PSL Research University, France\\
  \texttt{francis.bach@inria.fr}
}
\date{}

\begin{document}
\maketitle

\begin{abstract}
 We consider linear recurrent neural networks, which have become a key building block of sequence modeling due to their ability for stable and effective long-range modeling. In this paper, we aim at characterizing this ability on a simple but core copy task, whose goal is to build a linear filter of order $S$ that approximates the filter that looks $K$ time steps in the past (which we refer to as the shift-$K$ filter), where $K$ is larger than $S$. Using classical signal models and quadratic cost, we fully characterize the problem by providing lower bounds of approximation, as well as explicit filters that achieve this lower bound up to constants. The optimal performance highlights an uncertainty principle: the optimal filter has to average values around the $K$-th time step in the past with a range~(width) that is proportional to $K/S$.

\end{abstract}

\section{Introduction}

Since their early development~\citep{rumelhart1986sequential, elman1990finding}, recurrent neural networks (RNNs) have advanced machine learning for sequential data, with milestones such as echo-state networks~\citep{jaeger2001echo} and LSTMs~\citep{hochreiter1997long}. However, two problems severely limit the application of classical RNNs in modern times: (1) GPU hardware optimized for large matrix operations struggles with efficient sequential processing, and (2) RNNs are notoriously difficult to train due to vanishing and exploding gradients~\citep{bengio1994learning, pascanu2013difficulty}. As a result, transformers~\citep{vaswani2017attention} have emerged as the dominant solution for sequence processing, offering desirable scalability properties and less challenging optimization. However, the attention mechanism powering transformers relies on computing pairwise interactions between inputs at each timestamp, resulting in a squared inference and memory complexity $O(L^2)$ in the sequence length $L$. Instead, classical RNNs require one pass through the data to recurrently update their hidden state, bringing their complexity down to $O(L)$. This property is particularly desirable in the long-context setting~(e.g., analysis of long documents or genomics).

Indeed, in the interest of efficiency, we have recently witnessed a \textit{resurgence of new RNNs} in state-of-the-art industry-size applications such as language modeling~\citep{gu2024mamba, peng2024eagle, qin2024hgrn2, de2024griffin, yang2024parallelizing}. Sparked from the S4 model~\citep{gu2022efficiently}, these new recurrences offer $O(L)$ complexity as classical RNNs, yet are parallelizable on modern hardware like attention. At the core of their efficiency is a simplified recurrence that is \textit{linear} in the hidden state:
\begin{equation}
    x_{n} = A_n x_{n-1} + B_n u_n, 
    \label{eq:1}
\end{equation}
where $u_n$ is the input data at timestamp $n$, $x_n$ is the hidden state (which is a linear combination of inputs $u_1,u_2,\dots, u_n$), and $A_n,B_n$ are input-controlled transition matrices with a special parametrization~\citep{orvieto2023resurrecting}. Compared to previous RNNs, $A_n$ and $B_n$ have\textit{ no dependency on the hidden state}---a feature which reduces expressivity~\citep{merrill2024illusion,cirone2024theoretical} but unlocks GPU-efficient processing~\citep{martin2018parallelizing, smith2023simplified}.

New linear RNNs offer improved inference complexity and competitive performance on language modeling tasks~\citep{dao2024transformers,waleffe2024empirical}, as well as state-of-the-art results on several other domains including vision~\citep{liu2024vmamba,li2025videomamba, liang2024pointmamba, xing2024segmamba}, audio generation~\citep{goel2022sashimi}, online learning~\citep{zucchet2023online}, reinforcement learning~\citep{lu2023structured} and genome analysis, where the $O(L)$ complexity can tackle long DNA sequences~\citep{nguyen2024sequence}.

Despite the practical advantages of new linear recurrent mechanisms, we are at a very early evaluation stage in regards to assessing and understanding the capabilities and optimization properties of such systems when compared to~(1) transformers and (2) non-linear~(classic) RNNs. While several works are devoted to establishing a direct connection between transformers and linear RNNs~\citep{katharopoulos2020transformers, schlag2021linear, ali2024hidden, sieber2024understanding}, others point to fundamental and drastic differences in regards to expressivity and basic capabilities. Further, despite~\citet{orvieto2024universality,wang2024state, cirone2024theoretical} provide infinite-width theoretical guarantees for the expressivity of deep architectures based on linear RNNs, other works focusing on specific reasoning tasks of general interest in language modeling tell a different story: \cite{arora2023zoology} identified in the problem of selective \textit{copying}~(i.e., of recalling a specific value from the past, when presented the relative key) a fundamental discrepancy between attention and RNNs: building up a memory of past inputs, as opposed to direct edges, can fundamentally limit finite-width performance of linear RNN based models. This finding inspired a formal investigation by~\cite{jelassi2024repeat}, who proved that perfectly retrieving~(i.e., with zero error) inputs from distant past requires the RNN width to increase linearly with the sequence length. This in contrast to attention-based models, that can build associative mechanisms to solve such tasks with 2 layers \citep[cf.][]{olsson2022context}.

Inspired both by the practical relevance of new linear RNNs and by the need of further theoretical investigations of their basic properties, in this work we mathematically investigate arguably the most basic long-range task: recalling inputs seen $K$ timestamps before the current processing step. Such task has a close relation to the \textit{copy task} by~\cite{jelassi2024repeat}, while being simpler and with a clear challenge: successful replay as $K$ increases. As~\citet{jelassi2024repeat}, we are specifically interested in characterizing optimal performance as a function of the recall range $K$ and the memory size $S$---the dimension of the hidden state $x$.  Yet, while~\citet{jelassi2024repeat} work in the finite-vocabulary input setting standard in language modeling, assuming no particular structure in the recurrence, we take instead a signal processing approach, which allows us to characterize in detail the tradeoff between long-memory requirements~(large $K$) and optimal recall resolution under reduced memory size~($S<K$). Since the task is independent from the input value to recall, we restrict our attention to the case where in Eq.~\eqref{eq:1}, $A_n$ and $B_n$ are input-independent and hence fixed matrices: $A,B$. Further, as common in modern RNNs, we consider without loss in generality\footnote{This equivalence is often used in linear systems theory~\citep{hespanha2018linear}. Let us start from $x_n = A x_{n-1} + Bu_n$. Over the space of \( S \times S \) non-diagonal real matrices $A$, the subset of those non-diagonalizable in the complex domain has measure zero~\citep{bhatia2013matrix}. Thus, with arbitrarily small perturbations, \( A = Q \text{diag}(a) Q^{-1} \). This implies $Q^{-1} x_n = \text{diag}(a) (Q^{-1} x_{n-1}) + (Q^{-1} B) u_n$. Renaming \( x_n \leftarrow Q^{-1} x_n \) and \( B \leftarrow Q^{-1} B \) yields a diagonal complex-valued recurrence. \label{diagfootnote}} the diagonal case $A = \text{diag}(a)$. For one-dimensional input sequences and a final sum operation, if the RNN is initialized with zero memory, the scalar output sequence $(y_n)_{n \geq 0}$ can be computed through a \textit{convolution}:
\begin{equation}
    x_n \!=\! \text{diag}(a) x_{n-1} + u_n b,\quad y_n \!=\! 1^\top x_n  \ \ \implies \ \ y_n = (c \ast u)_n = \sum_{k=0}^n c_k u_{n-k} \quad \text{with} \quad c_k \!=\! \sum_{s=1}^S a^k_s b_s.
    \label{eq:filter_new}
\end{equation}
The task then consists in finding potentially complex vectors $a,b$ such that $y_n \approx u_{n-K}$. This is equivalent to requiring the sequence $(c_k)_{k \geq 0}$ to approximate the \textit{shift-$K$} filter $d = \delta_{K}$ (which is a sequence that is zero everywhere except at position $K$, that is, $d_k = 1_{k=K}$). 

In order to assess the approximation of $d = \delta_K$ by $c$ in the form of Eq.~(\ref{eq:filter_new}), we consider the idealized situation of infinite-length random stationary signals $(u_n)_{n \in \mathbb{Z}}$, and consider the expected loss function at time $n=0$, $\mathbb{E} \big[  | (c \ast u)_0 - (d \ast u)_0|^2 \big]$, where the expectation $\mathbb{E}$ is taken with respect to the distribution of the random sequence $(u_n)$. By stationarity of $(u_n)$ and the law of large number, this is equivalent to the mean-square-error over the entire sequence:
\begin{equation}
\label{loss}
\mathbb{E} \big[  | (c \ast u)_0 - (d \ast u)_0|^2 \big] = \lim_{N \to +\infty} \frac{1}{2N+1} \sum_{n = - N}^N  | (c \ast u)_n - (d \ast u)_n|^2
.\end{equation}

We study this loss function for $u$ being the white noise (problem becomes $\min_{a,b}\| c(a,b) - d \|_2^2$), and for simple auto-correlations $\E[u_k \bar{u}_{k'}] = \rho^{|k-k'|}$ for $\rho \in [0,1)$ ($\rho=0$ corresponding to white noise).

\paragraph{Contributions.} We make the following contributions:
\begin{enumerate}
    
    \item We provide a lower bound 
    on the best possible for the shift-$K$ loss above~(optimized with respect to $a$ and $b$) using tools from the approximation of rational functions~\citep{baratchart2016minimax} and Cauchy matrices~\citep{yang2003generalized}. For white noise, we obtain the lower bound $1- \frac{S}{K}$, showing that a large copy lag $K$ leads to an increase in error. This is made more precise with more general $\rho$'s, with the lower bound $\big(1 - \frac{3S}{K} \frac{1}{1-\rho} \big)_+$, showing that a small error can be obtained for autocorrelated input signals.

    \item We find an analytical solution to the shift-$K$ problem close to our lower bound (with matching behavior up to constants, and thus nearly optimal). Our solution allows us to instantiate an uncertainty principle, providing a clear intuition on resolution/memory tradeoffs~(see Fig.~\ref{figure uncertainty principle}). In addition, our closed-form solution for $a$ in Eq.~\eqref{eq:filter_new} allows us to motivate from a task-specific memorization perspective the successful S4D-Lin initialization by~\citet{gu2022parameterization} --- the simplest linear RNN initialization allowing to solve the most challenging tasks in the long-range arena~\citep{tay2020long}.
\end{enumerate}

The loss of our near-optimal solution illustrates the trade-off between recall range~($K$), memory size~($S$) and recall precision~(i.e., the concentration of the filter $c$ in Eq.~\eqref{eq:filter_new} around the spike $\delta_{K}$). Surprisingly, our finding can be formulated as an uncertainty principle: 
\begin{center}
\textit{Learning a filter~$c$ centered around a large $K$ for a fixed state-size $S$ is relatively easy,\\ yet increasing time-horizon $K$ comes at the expense of resolution~(width\footnote{For a filter designed to approximate the shift-$K$, we call width the width at the halfway height of the peak centered in $K$. The narrower the width, the better the approximation of the shift-$K$. The width can be interpreted as the resolution of the filter.} of the filter). }
\end{center}
As illustrated in Fig.~\ref{figure uncertainty principle}, perfect recall is eventually achieved at $S=K$. For $K>S$, the width of the filter around the correct location is proportional to $K/S$.

\begin{figure}[h]
    \centering
    \includegraphics[width=.95\linewidth]{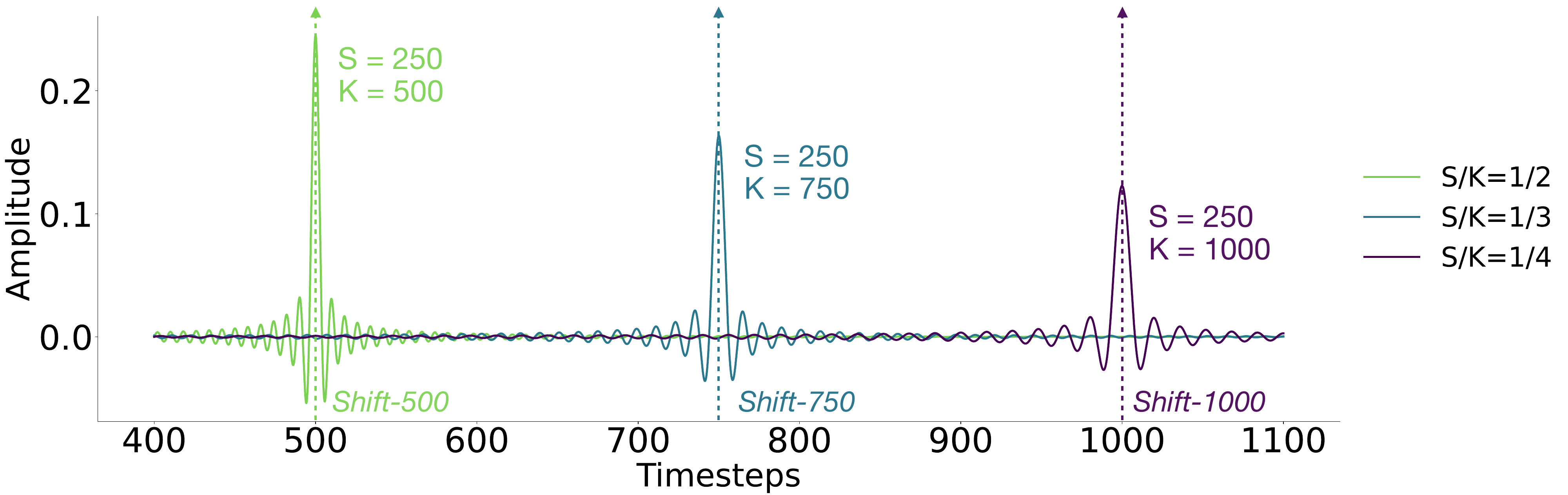}

    \vspace*{-0cm}
    
    \caption{\textit{Learning to shift-$K$ with linear recurrences exhibits an uncertainty principle. For fixed $S=250$, different values of $K$ induce different performances: the smaller the ratio $S/K$, the lower the peak of the filter and the larger the width. For a fixed memory size $S$, increasing the time horizon is feasible, but comes at the expense of resolution. For $K>S$, the width of the filter around the correct location is $K/S$.}}
    \label{figure uncertainty principle}
\end{figure}

\section{Notation and Main Results}
In this paper, we operate in the complex domain as usually done in the literature on SSMs~\citep{gu2022parameterization,orvieto2023resurrecting}. The reason for this choice in the literature is motivated by good performance, increased expressivity guarantees~\citep{orvieto2024universality,ran2024provable}, and most of all by the equivalence between dense linear RNNs and diagonal complex-valued RNNs$^{\ref{diagfootnote}}$.

\paragraph{Time domain.} Starting from Eq.~\eqref{eq:filter_new}, given a generic real-valued input $u=(u_n)_{n\in\mathbb{Z}}$, the output $y=(y_n)_{n\in\mathbb{Z}}$ of a linear RNN with parameters $(a,b) \in \mathbb{C}^S \times \mathbb{C}^S$ can be written as 
$y_n = (c\ast u)_n:= \sum_{k=0}^\infty c_k u_{n-k}$, where the convolution kernel $c = (c_k)_{k\in\mathbb{N}}$ is defined by $(a,b)$  as: 
\begin{equation}
\label{eq:cab}
c_k = \sum_{s=1}^S a^k_s b_s,
\end{equation}
for all $k\in\mathbb{N}$. Let $d = (d_k)_{k\in\mathbb{N}}$ be a second convolution kernel processing the input---the one we would like to approximate with our RNN (that is, $d_k = 1_{k = K}$). One can compute the expected squared norm between outputs of $d\ast u$ and $c\ast u$ at only a single $n \in \mathbb{Z}$ (as shown in Eq.~\eqref{loss}, this is also the mean-square error over the entire sequence):
\begin{align*}
     \mathbb{E} \Big[ |(d\ast u)_n-(c\ast u)_n|^2  \Big] &=   \mathbb{E} \Big[ \Big|\sum_{k=0}^\infty (c_k -d_{k}) u_{n-k} \Big|^2 \Big] =  \sum_{k,k'=0}^\infty (c_k -d_{k}) (\bar{c}_{k'} -\bar{d}_{k'})  \mathbb{E} [ u_{n-k} u_{n-{k'}}].
\end{align*}
Using stationarity of the signal $u$, $\mathbb{E} [ u_{n-k} u_{n-{k'}}] = \gamma(k-k')$ only depends on $k-k'$, and, we get our objective function, to be optimized with respect to the RNN parameters $(a,b)$:
\begin{equation}
    \mathcal{L}_{\text{time}}(c, d) = \sum_{k, k'=0}^{\infty}(c_k - d_k)(\bar{c}_{k'} - \bar{d}_{k'})\gamma(k-k')
    \label{Time domain loss},
\end{equation}
where \(\gamma(k - k')\) is the auto-correlation function that captures average temporal dependencies, weighting the contribution of errors based on time step correlations~\citep{brockwell2002introduction}. When there is no ambiguity on the filters $(c_k)_{k \in \mathbb{N}}, (d_k)_{k \in \mathbb{N}}$, we will refer to the loss $\mathcal{L}_\text{time}(c, d)$ as $\mathcal{L}_\text{time}$. We adopt the common choice \(\gamma(k) = \rho^{|k|}\) with \(\rho \in [0,1)\), also used recently by~\citet{zucchet2024recurrent}, where \(\rho = 0\) corresponds to uncorrelated white noise, 
where 
$\mathcal{L}_{\text{time}}(c, d) = \sum_{k =0}^{\infty}|c_k - d_k|^2$,
and \(\rho \to 1\) reflects strong temporal dependencies.

\paragraph{Frequency Domain.} In this work, we aim at approximating the action of the shift-$K$ filter $d = \delta_{K} := (1_{k = K})_{k\in\mathbb{N}}$. 
We find it convenient to process the loss above in frequency domain. The discrete-time Fourier transforms and Parseval's theorem allow to write the loss in Eq.~\eqref{Time domain loss} as
\begin{equation}
\mathcal{L}_\text{freq}(C, D) = \frac{1}{2\pi}\int_{-\pi}^\pi\vert C(e^{i\omega}) - D(e^{i\omega})\vert^2\Gamma(e^{i\omega})d\omega,
    \label{Frequential loss copy task}
\end{equation}
where $C(e^{i\omega}) = \sum_{s=1}^S\frac{b_s}{1-a_se^{-i\omega}}$ is a rational function of $e^{-i\omega}$, $D(e^{i\omega}) = e^{-iK\omega}$~(DFT of a shifted Dirac impulse) and $\Gamma(e^{i\omega}) = \frac{1-\rho^2}{\vert 1 - \rho e^{-i\omega}\vert^2}$, thus turning the problem to that of rational approximations on the unit circle~\citep{baratchart2016minimax}. See Appendix \ref{appendix subsection natural pair} and \ref{appendix subsection frequency loss} for more details. When there is no ambiguity on the Fourier transforms $C$ and $D$, we will refer to the loss $\mathcal{L}_\text{freq}(C, D)$ as $\mathcal{L}_\text{freq}$.

\paragraph{Overview.} In Section \ref{section lower bound}, we provide a lower bound on $\mathcal{L}_{\text{time}}$ suggesting that having a small state size does not necessarily imply a short memory capacity; however, the bound also shows that this comes at the cost of a degraded resolution. This provides a first connection with our uncertainty principle in the context of learning shifts with linear models and reveals a fundamental tradeoff between the time horizon of our copy task and the performance of the filter, given a fixed size of the model. Furthermore, we highlight the significant role of data autocorrelation, demonstrating that while linear models struggle to retain white noise, their performance improves substantially when dealing with autocorrelated data, which may better reflect real-world scenarios.
In Section~\ref{section upper bound}, we establish our uncertainty principle by deriving a closely matching upper bound. To do this, we consider the loss $\mathcal{L}_\text{freq}$ to carefully design a new filter that performs similar to the lower bound, up to a constant factor. This representation, providing meaningful results in practice, gives insights on the behavior of linear RNNs as they implement longer memory.

To summarize, our insights stem from two main results. The first one is a lower bound on the best possible error greater than  $1- \frac{S}{K}$ for $\rho=0$ and $\big(1 - \frac{3S}{K} \frac{1}{1-\rho} \big)_+$ for $\rho \in [0,1)$ (see Theorems~\ref{lower bound white noise} and~\ref{theorem 
autocorrelated lower bound}). The second is an upper bound (construction of an explicit filter) that matches the lower-bound up to a constant factor~(thus establishing\footnote{Our uncertainty principle, as formulated in the introduction, is first suggested by our lower bound but only formally implied by~(and immediately follows from) the combination with our matching upper bound.} our uncertainty principle), as informally described below in Theorem~\ref{thm:informal_upper} and illustrated in Fig.~\ref{fig:window}.

\begin{figure}[H]
\vspace{-1mm}
    \centering
    \includegraphics[width=.85\linewidth]{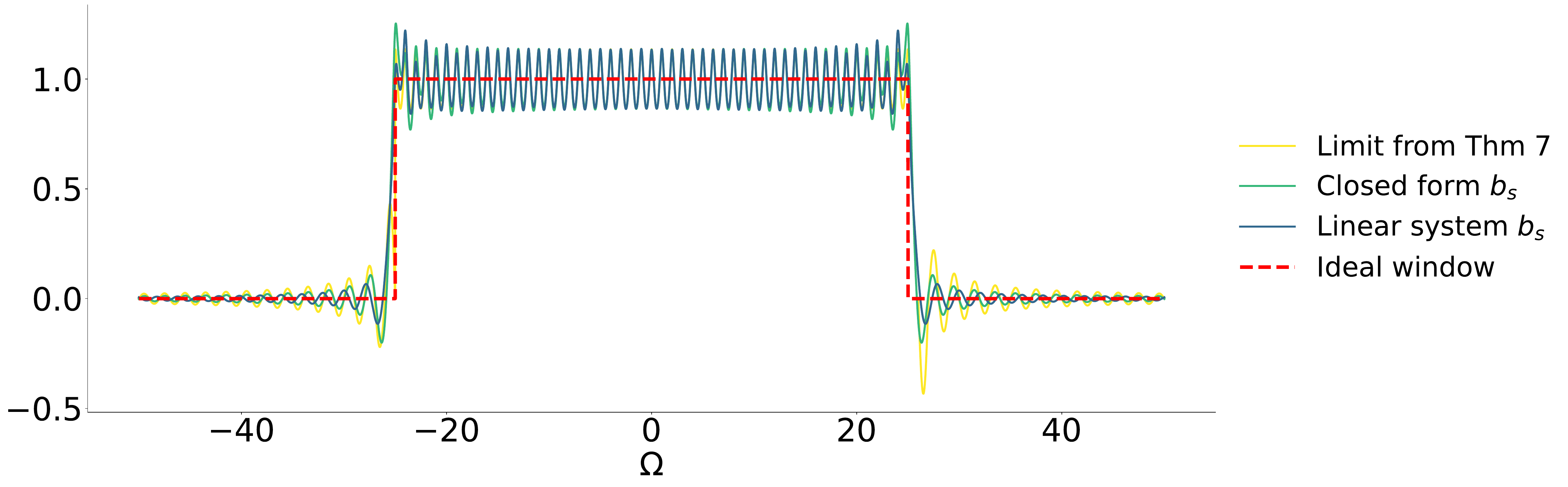}

    \vspace*{-.2cm}
    
    \caption{\textit{Shown is the behavior of $\frac{C(e^{i\omega})}{D(e^{i\omega})}$, where $C(e^{i\omega})$ is the Fourier transform of our near-optimal filter in Theorem~\ref{thm:informal_upper} and $D(e^{i\omega}) = e^{-iK\omega}$ is the Fourier transform of our Shift-$K$ filter. Perfect match between filters implies the ratio is $1$ for all $\omega$.  If instead this equality holds in a window, then the filter would effectively act as a Shift-$K$ for inputs with frequencies $\Gamma(e^{-i\omega})$ in the same window. For $S=51, K=500$, we denote $T = \frac{S-1}{2}$ and plot the ratio $\frac{C(e^{i\omega})}{D(e^{i\omega})}$ with respect to $\Omega=\frac{K\omega}{\pi}$ (to dilate the space). The asymptotic ratio $\frac{C(e^{i\omega})}{D(e^{i\omega})}$ (yellow) from Theorem \ref{convergence to window}, the same ratio for linear models with ($b_s$) given by Eq.~\eqref{Param_of_the_bs} (green), and for ($b_s$) given by linear system inversion Eq.~\eqref{bs linear system inversion} (blue) are compared. The model effectively approximates the shift-$K$ operation, within the frequency window $[-\frac{\pi T}{K}, \frac{\pi T}{K}]$, while vanishing outside this window, leading to a time resolution~(inverse of filter width) of $\frac{S}{K}$. This behavior underscores the uncertainty principle associated with the filter: for small $S/K$ ratios and uncorrelated data, the approximation holds over a narrow frequency range. As autocorrelation increases, the approximation domain shrinks, enhancing accuracy. In red, we show the perfect window~(value of 1 on $[-\frac{\pi T}{K}, \frac{\pi T}{K}]$ and~$0$ outside).}
}
    \label{fig:window}
\end{figure}

\begin{theorem}[upper bound, informal]
Let $S$ be odd and $T = \frac{S-1}{2}$. The filter defined by Eq.~\eqref{eq:cab} with $a_s = \exp\left(-\frac{\alpha}{K}\right)\exp\left(i\frac{\pi s}{K}\right)$ and  $b_s \propto (-1)^s$ for $ s \in \llbracket -T, T\rrbracket$\footnote{While in our introduction, for clarity, we considered $s\in\llbracket1,S\rrbracket$, hereafter, for ease of notation and to emphasize symmetry in our filter, we reparametrize $s\in\llbracket -T,T\rrbracket$ where $T = \frac{S-1}{2}$. The dimension of our hidden state in Eq.~\eqref{eq:filter_new} remains $S$.},   achieves an approximation error comparable to the lower bound up to a constant factor in the context of white noise data. This is because our filter accurately approximates a shift-$K$ in frequency domain over the window \(\left[- \frac{\pi T}{K}, \frac{\pi T}{K} \right]\), vanishing outside this range.
\label{thm:informal_upper}
\end{theorem}

The connection between Theorem~\ref{thm:informal_upper} and HiPPO initialization is presented in Sec.~\ref{sec:hippo}

\section{Related Works}
\label{rw}

\textbf{Attention and RNNs.} In order to reduce the $O(L^2)$ complexity burden in transformers, techniques such as patching~\citep{dosovitskiy2021image,pagnoni2024byte}, gradient checkpointing~\citep{chen2016training}, and FlashAttention~\citep{dao2022flashattention, dao2023flashattention} become crucial when training and deploying models at scale. Despite this limitation, transformers successfully power most state-of-the-art architectures we use today: beyond large language models~\citep{brown2020language, team2024gemini}, attention found widespread application in vision~\citep{dosovitskiy2021image}, graphs~\citep{ma2023graph} and DNA~\citep{dalla2024nucleotide}.  Nevertheless, the quadratic complexity of attention has remained a pressing limitation, prompting numerous efforts over the years to develop efficient approximations~\citep{wang2020linformer, choromanski2020rethinking, chen2021skyformer, lee2022fnet} that inevitably bring attention closer to RNNs~\citep{schlag2021linear,katharopoulos2020transformers}. Indeed, more recently, we have witnessed a resurgence of RNNs in state-of-the-art industry-size applications such as language modeling. Sparked by S4~\citep{gu2020hippo, gu2022efficiently}, which surpassed attention on long-range reasoning tasks~\citep{tay2020long}, we have seen a drastic increase in the usage of RNNs in deep architectures, albeit in a linear form that guarantees both $O(L)$ memory/inference complexity and fast computation on GPUs~\citep{martin2018parallelizing} while matching or surpassing transformers on downstream tasks: a prime example are state-space models~(SSMs) such as Mamba~\citep{gu2024mamba}, along with architectures based on RNNs~\citep{de2024griffin, peng2024eagle, yang2024gated}. 
\\[.35cm]
\textbf{Special initialization of SSMs.} While in natural language SSMs are relatively robust to initialization, on challenging long-range reasoning or memorization tasks careful initialization is crucial~\citep{gu2023how,orvieto2023resurrecting,trockman2024mimetic}. The used schemes stem from the HiPPO theory by~\citet{gu2020hippo}: a special initialization can provably construct features related to the coefficients of optimal polynomial approximations of the input signal. Despite this intriguing connection, already S4~\citep{gu2022efficiently}, the first SSM, deviates quite significantly from the HiPPO prescription. Latest initialization such as S4D-Lin~\citep{gu2022parameterization} or the one by~\citet{orvieto2023resurrecting} are only vaguely related to the HiPPO and present a single non-trivial property: recurrent eigenvalues~($a$ in Eq.~\eqref{eq:filter_new}) are complex-valued, with coupled phase and magnitude. Our theory provides a formal justification of this choice from a memorization perspective.
\\[.35cm]
\textbf{Theoretical guarantees for (non)-linear RNNs.} Expressivity of standard nonlinear RNNs has been extensively studied from a Turing completeness perspective~\citep{siegelmann1992computational,korsky2019computational}. Taking instead the signal processing angle,~\citet{hanson2020universal} proved that wide enough non-linear RNNs can approximate up to vanishing precision non-linear time-homogeneous systems of differential equations driven by input paths. The argument used here is based on Barron's theorem~\citep{barron1993universal} for approximation of continuous functions with neural networks with one hidden layer. Regarding instead linear RNNs such as Eq.~\eqref{eq:1}, results are more recent and have been mostly driven by deep learning developments. \citet{li2022approximation} showed that linear time-invariant RNNs~($A_n$ and $B_n$ independent of $n$, as in this paper) can approximate arbitrary convolution filters as the hidden state size $S$ grows to infinity. Further,~\citet{hanson2019universal} proved that stacking exponentially~(in the sequence length) many temporal convolution filters, chained together with ReLU activations, leads to approximation of arbitrary non-linear filters. Recent works~\citep{orvieto2024universality,wang2024state} prove the universality of linear recurrences~(one layer) when equipped with a fixed (timestamp independent) point-wise MLP acting across the recurrence output, with intriguing connections to Volterra series~\citep{boyd1985fading}. Finally, expressivity of latest models such as Mamba has been studied by~\citep{cirone2024theoretical}. Further, language-specific capabilities of new SSM and RNNs have been studied by~\citet{merrill2024illusion}~(state tracking) and~\citet{jelassi2024repeat}~(copying). 

\section{Lower Bound}\label{section lower bound}

We aim to establish a lower bound on the approximation error $\mathcal{L}_\text{time}(c,d)$ where $c$ has the RNN form in Eq.~\eqref{eq:cab}. By deriving this lower bound, we provide a theoretical benchmark for evaluating the effectiveness of linear time-invariant filters in our shift-$K$ task. Importantly, we demonstrate that the derived lower bound depends on the ratio \(\frac{S}{K}\), where \(S\) represents the hidden dimension and \(K\) is the horizon of our copy task.

To gain deeper insights into the performance of these filters, we analyze the approximation error in two scenarios: the case of white noise ($\rho=0$), and the case of autocorrelated data $(\rho>0)$.

\subsection{White Noise}

With white noise input $\mathcal{L}_\text{time}(c, d)$ has a simpler squared $\ell_2$-norm formulation:

\begin{equation}
\mathcal{L}_\text{time}(c, d)= \sum_{k=0}^{+\infty}\vert c_k - d_k\vert^2 = 1 + \sum_{k=0}^{+\infty}\vert c_k\vert^2 - 2{\textnormal{ Re}}\big(\sum_{k=0}^{+\infty}c_kd_k\big),
\label{white noise time domain loss}
\end{equation}
where we recall that $c$ has the form Eq.~\eqref{eq:cab} and that the shift-$K$ filter $d_k= 1_{k=K}$ has norm one. The following theorem shows a lower bound.

\begin{theorem}[Lower bound of the approximation error---white noise]\label{lower bound white noise}
 Let $S$ and $K$ be two  positive integers. The approximation error $\mathcal{L}_\text{time}(c, d)$ of the shift-$K$ filter $d$ by a filter $c$ of the form in Eq.~\eqref{eq:cab} is lower bounded by $1-\frac{S}{K+1}$.
\end{theorem}

\begin{proof}(Sketch, see full proof in Appendix~\ref{appendix subsection white noise loss}).
Given the form of $c$ as $c_k = \sum_{s=1}^S b_s a_s^k$, the loss in Eq.~\eqref{white noise time domain loss} has an explicit expression by summing geometric series over $k$, leading to:
\begin{equation}
    \mathcal{L}_\text{time}(c, d)= 1 + \sum_{s, s'=1}^S\frac{b_s\bar{b}_{s'}}{1-a_s\bar{a}_{s'}} - 2{\textnormal{ Re}}\Big(\sum_{s=1}^Sb_sa_s^K\Big).
    \label{good as and bs loss}
\end{equation}
 We thus want to maximize with respect to $a_s$, $b_s$, $s\in\llbracket 1, S\rrbracket$, the following quantity:
\begin{equation*}
 2{\textnormal{Re}}\big(\sum_{s=1}^Sb_sa_s^K\big) - \sum_{s,s'=1}^S\frac{b_s\bar{b}_{s'}}{1-a_s\bar{a}_{s'}},
\end{equation*}
which is equal to 
\begin{equation}
\label{Q}
\langle \bar{b}, a^K\rangle + \langle a^K, \bar{b}\rangle - \langle \bar{b}, C\bar{b}\rangle,
\end{equation}
where $C$ is an $S \times S$ matrix with entries $\displaystyle C_{ss'}=\frac{1}{1-a_s\bar{a}_{s'}}$, and $\langle \cdot,\cdot\rangle$ is the standard Hermitian product. This is a quadratic form in $b$, and thus we can maximize with respect to $b$ in closed form, leading to the performance criterion  $ \mathcal{L}_\text{time} = 1 - F_K$, with 
\begin{equation}
\label{eq:FK}
F_K = \langle a^K, C^{-1}a^K \rangle = \sum_{s,s'=1}^{S} \bar{a}_s^K (C^{-1})_{ss'} a_{s'}^K.
\end{equation}
This is a function of the $a_s$'s only since we have maximized out the $b_s$'s. This function is rational but has a complicated expression.
In order to bound it, we notice that the matrix $C$ has some ``displacement structure'' similar to Cauchy matrices~\citep{yang2003generalized},
that is, 
$$
C - \Diag( {a}) C \Diag(\bar{a})  = 1_S 1_S^\top,
$$
which leads to, after some manipulations, to a ``closed form'' expression for the inverse $C^{-1}$:
$$
(C^{-1})_{ss'} \big( \frac{1}{ \bar{a}_s a_{s'}}-1 \big) = u_s \bar{v}_{s'},
$$
with $ u = C^{-1} 1_S$ and $v = \Diag( \bar{a})^{-1}  C^{-1} \Diag(a)^{-1} 1_S \propto u$. Moreover, the vector $u$ happens to have a simple characterization through rational functions as
$$
\sum_{s=1}^S \frac{ u_{s}}{1 - z \bar{a}_{s}} = 1 - \prod_{s=1}^S \bar{a}_{s} \prod_{s=1}^S \frac{ a_{s} - z}{1 - z \bar{a}_{s}}.
$$
This allows to characterize the Fourier series of $F_K$ and get an explicit bound using properties from rational approximations on the unit circle~\citep{baratchart2016minimax}. See details in Appendix~\ref{appendix subsection white noise loss}.
\end{proof}

The lower bound established in Theorem~\ref{lower bound white noise} demonstrates that the approximation error remains close to 1 when the ratio \(S/K\) is small. This highlights the inherent difficulty of approximating shift-$K$ filter using linear recurrences in this regime. Nevertheless, by increasing the dimension of the parameters $S$, with fixed $K$, we can hope to achieve a better loss. This shows a fundamental tradeoff in the linear model's ability to solve the copy task, in the context of white noise -- connected to our uncertainty principle. Allowing auto-correlated signals gives a finer picture, this is explored in the next subsection.

\subsection{Autoregressive Autocorrelation}

In this context, we consider a non-zero correlation factor defined as \(\gamma(k) = \rho^{\vert k\vert}\) to account for the temporal structure in the data. This approach with $\rho>0$ simulates situations with real-life data, whose autocorrelation is often modeled this way.~\citep{brockwell2002introduction}. The loss function writes in this case:
\begin{equation}
    \mathcal{L}_\text{time}(c, d) = \sum_{k,k'=0}^{+\infty}(c_k-d_k)(\bar{c}_{k'}-\bar{d}_{k'})\rho^{\vert k-k'\vert},
    \label{correlated time domain loss}
\end{equation}
where \(c_k\) is defined as in Eq.~\eqref{eq:filter_new}, and \((d_k)\) is given by $d_k = 1_{k=K}$. The following theorem extends Theorem~\ref{lower bound white noise}
to all $\rho$'s. We use the notation $(y)_+=\max(y, 0)$.

\begin{theorem}[Lower bound of the approximation error---auto-correlated noise]\label{theorem 
autocorrelated lower bound}
Let $S$ and $K$ be two integers. The approximation error $\mathcal{L}_\text{time}(c, d)$ of the shift-$K$ filter $d$ by a filter $c$ of the form in Eq.~\eqref{eq:cab} is lower bounded by, for the autoregressive autocorrelation $\Big(1-\frac{S}{K}\frac{3}{1-\rho}\Big)_+$.
\end{theorem}

\begin{proof}(Sketch, see full proof in Appendix~\ref{proof auto}).
Let \((c_k)\) be a linear-time filter such that \(c_k = \sum_{s=1}^S b_s a_s^k\), and let \((d_k)\) be defined as $d_k =  1_{k=K}$. Let $w_s = b_s a_s / ( a_s - \rho)$, for $s \in \llbracket 1,S \rrbracket$. We can compute $\mathcal{L}_\text{time}(c, d)$ in closed form by explicit summation leading to
\[
\mathcal{L}_\text{time}(c, d)=
1 - 2(1 - \rho^2)\operatorname{Re}\Big(\sum_{s=1}^{S+1}\frac{w_s a_s^K}{1-a_s\rho}\Big) + (1-\rho^2)\sum_{s,s'=1}^{S+1}\frac{w_s \bar{w}_{s'}}{1-a_s\bar{a}_{s'}},
\]
where we have used the specificity of the auto-correlation function to create a new pole, defined as \(a_{S+1} = \rho\) and the constraint \(\sum_{s=1}^{S+1} w_s a_s^{-1} = 0\) holds for some vector $w$ to be optimized. The minimum with respect to $w$ with the constraint is greater than the unconstrained minimizer, equal to (using the fact that this is a quadratic problem):
$ \displaystyle 
H_K = 1 - (1-\rho^2)\sum_{s, s'=1}^{S+1}\frac{\bar{a}_s^K}{1-\bar{a}_s\rho}\frac{a_{s'}^K}{1-a_{s'}\rho}(C^{-1})_{ss'},
$
where $C_{ss'} = \frac{1}{1-a_s\bar{a}_{s'}}$, is a matrix of a similar form as in the proof of Theorem~\ref{lower bound white noise}. The proof then follows similarly by using explicit expressions of matrix inverses.
\end{proof}

Therefore, in the autocorrelated case, $\mathcal{L}_\text{time}$ exhibits a lower bound that depends on the ratio \(\frac{S}{K}\). The error diminishes as \(\rho\) approaches 1, indicating that memorization may become more effective in the limit of strong autocorrelation. This behavior also suggests that memorization performance is intrinsically linked to the spectral characteristics of the data. Specifically, linear filters are incapable of precisely solving the copy task for time horizons~\(K\) larger than \(S\) in the presence of poorly correlated data, as in white noise. However, reducing the spectral domain, by imposing autocorrelation in the data, concentrates the signal's energy within specific frequency bands and significantly improves performance. Next, we further investigate this behavior by designing an explicit filter.

\section{Upper Bound}\label{section upper bound}

In this section, we complement the previous results, which showed that the lower bound of $\mathcal{L}_\text{time}(c, d)$ for the copy task using linear systems depends on the ratio \(\frac{S}{K}\). We present a closed-form parameterization of the filter that achieves a similar performance differing only by a constant factor. This parameterization serves as an upper bound on the achievable approximation accuracy of linear RNNs on the shift-$K$ task. In particular, we provide explicit expressions for the learnable parameters \(a_s\) and \(b_s\), accompanied by a theoretical analysis of their performance. This formulation establishes a theoretical upper limit for the smallest attainable error and highlights the behavior of a ``good'' filter. Since this upper limit also depends on the ratio \(\frac{S}{K}\), we can infer conclusions about the optimal behavior of the filter, establishing our uncertainty principle and particularly its relation to the spectral width of the data. 
To present our intuitions and results, we convert the time-domain loss $\mathcal{L}_\text{time}$ in Eq.~\eqref{Time domain loss} into its frequency-domain counterpart $\mathcal{L}_\text{freq}$ in Eq.~\eqref{Frequential loss copy task}.

\subsection{Parameterization of the Filter}

Here, we introduce a new filter inspired by the frequency representation of the problem, which achieves promising results on the copy task. $\mathcal{L}_\text{freq}(C, D)$ in Eq.~\eqref{Frequential loss copy task} suggests that a good filter $(c_k)$, denoted by $C(e^{i\omega})$ in the frequency domain, should approximate as best as possible the complex exponential~$e^{-iK\omega}$ for~$\omega~\in~[-\pi,\pi]$.
Additionally, Sec.~\ref{section lower bound} demonstrated that it is not possible to achieve an error smaller than \(1-\frac{S}{K}\) when solving the copy task using linear models on white noise.  Based on these results, we build a greedy approach, where each individual term~\(\frac{b_s}{1-a_se^{-i\omega}}\) of~$C(e^{i\omega})$ captures a single oscillation of the complex exponential. This should result in an error that depends on \(\frac{S}{K}\), and motivates the following representation.

\paragraph{Parameterization of the $\boldsymbol{a_s}$.} 
The parameters \(a_s\) govern the filter's ability to refer to earlier time steps, making them the most critical components of the recurrence. To provide finer control around  the complex unit circle, we employ an exponential parameterization, for $S$ odd:
\begin{equation}
    a_u = \exp\left(-\frac{\alpha}{K}\right)\exp\left(i\frac{\pi s}{K}\right), \quad s \in \llbracket -T, T\rrbracket, \text{ with } S = 2T+1,
    \label{Param_of_the_as}
\end{equation}
where \(0<\alpha\ll K\) so that $\vert a_s\vert<1$ (for stability of the system). The \(a_s\)'s have a constant modulus defined by the parameter \(\alpha\), while their phases are uniformly distributed around the unit circle, separated by an angular distance of \(\frac{\pi}{K}\).

\paragraph{Remark:} This representation in Eq.~\eqref{Param_of_the_as} ensures that the majority of the weight in each individual term \(\frac{b_s}{1-a_se^{-i\omega}}\) is concentrated around the frequency \(\frac{\pi s}{K}\), effectively capturing a single oscillation of the complex exponential. Our goal is to fit \(S\) oscillations of \(e^{-iK\omega}\), which would result in a loss proportional to \(\frac{S}{K}\).

\paragraph{Parameterization of the $\boldsymbol{b_s}$.} We can obtain the $b_s$'s by an approximate minimization as follows:
\begin{lemma}\label{Lemma param of bs}
    Let the parameters $a_s$ of the filter be defined as in Eq.~\eqref{Param_of_the_as}, where $\alpha$ is a positive real number. The asymptotic optimal parameters (when $K\rightarrow+\infty$) $b_s$ that minimize $\mathcal{L}_\text{freq}$ are given by: 
    \begin{equation}
        b_s = \frac{e^{-\alpha}(e^{2\alpha} - e^{-2\alpha})}{2K}(-1)^s, \quad s \in \llbracket -T, T\rrbracket.
        \label{Param_of_the_bs}
    \end{equation}
\end{lemma}

\begin{proof}
As highlighted in Eq.~\eqref{Q}, the approximation error, defined in terms of $a_s$ and $b_s$, is quadratic and convex with respect to $b_s$. Hence, the optimal solution is given by:
\begin{equation}
b = C^{-1}\bar{a}^K,
\label{bs linear system inversion}    
\end{equation}
where the matrix $C$ is defined as $C_{ss'} = \frac{1}{1 - a_s \bar{a}_{s'}}$. Using asymptotic expansions for large $K$, the eigenvector of $C$ corresponds to $z = \left((-1)^s\right)_s \in \mathbb{R}^S$ associated to the eigenvalue $\frac{2}{e^{2\alpha}-e^{-2\alpha}}$, yielding the result, thanks to the asymptotic expansion of $(a_s)$. See full proof in Appendix~\ref{appendix subsection asymptotic bs}.
\end{proof}

Note that $(a_s)$ and $(b_s)$ from complex conjugate pairs; this allows to obtain a real filter.

\begin{figure}[!htb]
    \centering
    \includegraphics[width=1\linewidth]{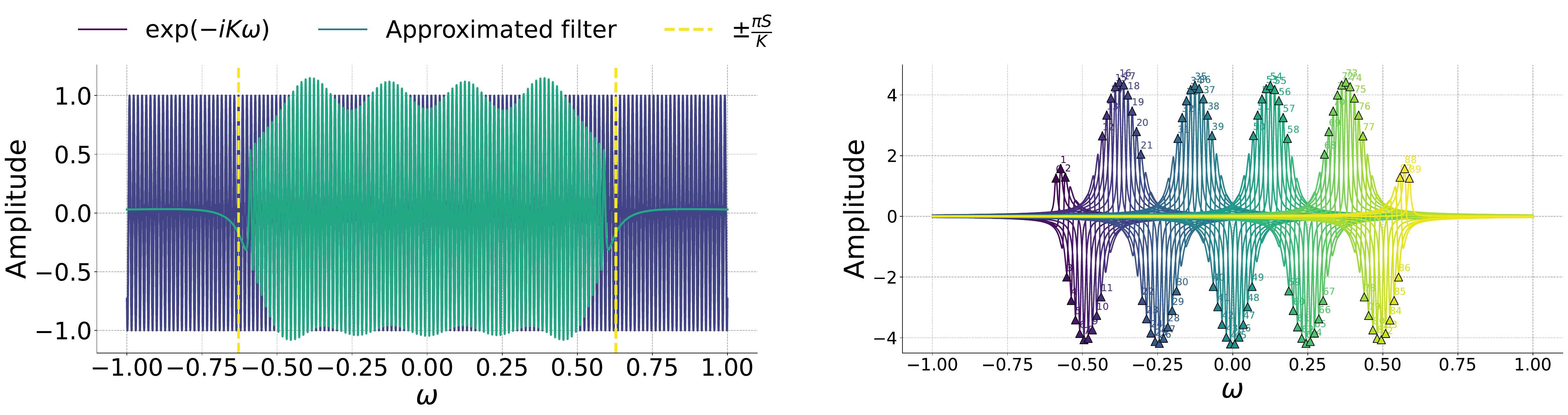}

    \vspace*{-.3cm}
    
    \caption{\textit{Poor performance of the filter for white noise data is due to its approximation of the complex exponential over a limited frequency window of size $\frac{\pi S}{K}$. Left: The target filter $\exp(-iK\omega)$ (blue) for $K=450$, and the approximated filter using linear recurrences (green) for $S=90$. The approximation is reasonably accurate within the frequency window of size $\frac{\pi S}{K}$, indicated by the dashed yellow  lines. Outside this window, the filter is zero, demonstrating the inability of filters based on linear recurrences to perfectly memorize long-range data with broad spectra. Right: Contributions from all individual terms $\frac{b_s}{1 - a_s e^{-i\omega}}$ for $s\in\llbracket -45, 45 \rrbracket$. Each individual term captures one oscillation of the complex exponential, making their contributions highly localized. This design reflects the structure of the filter's parameters.}}
    \label{figure spectral analysis}
\end{figure}

\begin{lemma}\label{definition new filter}
    Let $K$ and $S$ be two large integers such that $S \ll K$. Consider the parameters $(a_s)_{s \in \llbracket -T, T \rrbracket}$ from Eq.~\eqref{Param_of_the_as} and $(b_s)_{s \in \llbracket -T, T \rrbracket}$ from Eq.~\eqref{Param_of_the_bs}. Then the spectral representation of the filter is given by
    \begin{equation*}
        C(e^{i \omega}) = \sum_{s=-T}^{T} \frac{b_s}{1 - a_s e^{-i\omega}} = \sum_{s=-T}^{T} \frac{(-1)^s e^{-\alpha}(e^{2\alpha} - e^{-2\alpha})}{2K \big(1 - e^{-\frac{\alpha}{K}} e^{i(\frac{\pi s}{K}-\omega)}\big)}.
    \end{equation*}
\end{lemma}

In the time domain, this filter is approximately equivalent to a shifted sine cardinal, see Fig.~\ref{fig:time domain filter}, highlighting its inherent smoothness and symmetry. The positions of its parameters on the complex plane are strongly influenced by the ratio $S/K$, which corresponds to the horizon of the copy task relative to the order of the linear recurrence. This dependency captures the trade-off between long-term memory and the granularity of the recurrence structure.

\begin{theorem}[Upper bound of the error]\label{thm upper bound}
    Consider $c_k$ the filter defined in Lemma~\ref{definition new filter}, and $(d_k) = (1_{k=K})$ the shift-$K$ filter. Then, for $S,K \to +\infty$ with $S/K \to 0$, we have 
    \begin{equation}
        \mathcal{L}_\text{time}(c, d) \sim 1 - \frac{e^{-2\alpha}(e^{2\alpha} - e^{-2\alpha})}{2} \times \frac{S}{K}.
        \label{Upper bound as and bs}
    \end{equation}
\end{theorem}
\paragraph{Remark:} Note that the relation of $\mathcal{L}_\text{time}$ to $\alpha$ in Theorem~\ref{thm upper bound} incites to take $\alpha$ very large. Nevertheless, this would cause the value of the norm of $b_s$ in Eq.~\eqref{Param_of_the_bs} to explode. In practice, we always chose $\alpha =1$, a choice which also leads to a closer match with HiPPO initialization~(see Section~\ref{sec:hippo}). We further note that our bound on the $b_s$ in Lemma~\ref{Lemma param of bs} is strictly related to the discussion around benefits of complex parametrization in~\citet{ran2024provable}: as $a_s$ become closer to reals~(magnitude increases at equal phase), approximating arbitrary filters requires exploding coefficients~(cf. their Theorem 2). \ \\

When $S \ll K$, the term $\frac{S}{K}$ becomes very small, causing the error in Eq.~\eqref{Upper bound as and bs} to approach 1. We recover the result of Section~\ref{section lower bound}, obtaining a loss that is similar up to a constant factor. This approximation error for our filter serves as an upper bound for the approximation of shift-$K$ filter by linear recurrences.

This behavior can be attributed to the inherent properties of the filter, as illustrated in Fig.~\ref{figure spectral analysis}. The filter approximates reasonably well all the oscillations of $e^{-iK\omega}$ over the frequency window $[\frac{-\pi T}{K}, \frac{\pi T}{K}]$ and vanishes outside this window. Each individual term of the partial fraction decomposition is responsible for capturing a peak of the complex exponential. Therefore, data exhibiting large frequency spectrum like white noise cannot be memorized properly, explaining the poor performance of the filter on our copy task. This is how we designed it, to catch up with the lower bound. This is made precise in the following theorem.

\begin{theorem}\label{convergence to window}
     For $\alpha$ real and positive and $\Omega=\frac{K\omega }{\pi}$, 
     \[C(e^{i\omega}) = 
\sum_{s=-T}^{T} \frac{(-1)^u e^{-\alpha}(e^{2\alpha} - e^{-2\alpha})}{2K \big(1 - e^{-\frac{\alpha}{K}} e^{i(\frac{\pi u}{K}-\frac{\pi\Omega}{K})}\big)} \underset{S/K\rightarrow 0}{\underset{S\rightarrow+\infty}{\sim} }
\begin{cases} 
\frac{e^{-\alpha}(e^{2\alpha}-e^{-2\alpha})}{2}\times\frac{i(-1)^{T+1}\times 2\lfloor\Omega\rfloor}{2\pi(\lfloor\Omega\rfloor-T)(\lfloor\Omega\rfloor+T)} & \text{if } \vert\Omega\vert > T, \\
\frac{e^{-\alpha}(e^{2\alpha}-e^{-2\alpha})}{e^{\alpha}e^{i\pi\Omega}-e^{-\alpha}e^{-i\pi\Omega}} & \text{if } \vert\Omega\vert < T.
\end{cases}
\]
\end{theorem}
In particular, we obtain that $C(e^{i\omega})$ tends to $0$ if $\Omega$ is out of the window $[-T, T]$, while  inside the window, we obtain some oscillations around 1 whose magnitude depends on $\Omega$ (see Fig.~\ref{figure spectral analysis}).

When both $S$ and $K$ tend to infinity, $K$ being significantly larger than $S$, the filter converges to a rectangular window function on the frequency interval $[-\frac{\pi T}{K}, \frac{\pi T}{K}]$, taking the value oscillating around 1 within this interval and 0 outside it. See an illustration in Fig.~\ref{fig:window}. This limitation highlights why the filter performs poorly on white noise, as the uniform spectral density of white noise extends far beyond this narrow frequency window. Conversely, as the frequency window narrows (determined by the autocorrelation $\Gamma(e^{i\omega}))$, the filter becomes better aligned with the target response, leading to improved performance. 

\subsection{Performance in the Autocorrelated case}

The autocorrelation factor \(\Gamma(e^{i\omega})\) exerts a narrowing effect in the frequency domain, reducing the bandwidth of frequencies over which $\mathcal{L}_\text{freq}(C, D)$ is evaluated. See Appendix~\ref{appendix subsection natural pair} for more details. Since our filter is specifically designed to accurately approximate the oscillations of the complex exponential over a frequency window of size \(\frac{\pi S}{K}\), it follows logically that the loss decreases as the autocorrelation factor \(\rho\) approaches 1.

We can compute the loss of the idealized filter in the frequency domain $1_{|\omega| \leqslant \frac{2\pi S}{K}} e^{-i K \omega }$: 
$$
1 - \frac{1}{2\pi}\int_{-\frac{\pi S}{2K}}^\frac{\pi S}{2K} \Gamma(e^{i\omega}) = 
1 - \frac{2}{\pi}
\arctan \Big(
\frac{1+\rho}{1-\rho}\tan  \frac{\pi S}{K}
\Big) \sim 1 -\frac{2}{\pi}
\frac{1+\rho}{1-\rho}  \frac{\pi S}{K},
$$
when $S/K$ goes to zero, which is corresponding to the lower bound in Theorem~\ref{theorem 
autocorrelated lower bound}.

\subsection{Connection with HiPPO Initialization}
\label{sec:hippo}

HiPPO theory~\citep{gu2020hippo} was crucial for the development of modern recurrent models. The main result of this theory is that linear continuous-time ODEs~(linear RNNs, when discretized) can perform online compression of smooth input signals by storing projection onto an $S$-dimensional ~($S$ is the dimension of $x$ in Eq.~\eqref{eq:1}) polynomial basis. Starting from a \textit{dense} HiPPO-inspired $A$ matrix,~\citet{gupta2022diagonal} first proposed to initialize the $A$ matrix in Eq.~\eqref{eq:1} as the diagonal part of its ``diagonal plus low rank'' approximation. \citet{gu2022parameterization} additionally simplified this expression conjecturing~(see their Conjecture 5) a simplified closed-form solution that works well in practice: $a_s = \exp\left(-\frac{\Delta}{2}\right)  \exp\left(i\pi s\Delta\right)$~(S4D-Lin). The parameter~$\Delta$ here is a learnable coefficient resulting from discretization of the approximate HiPPO system. There is no theory indicating how to initialize this coefficient, though further studies~\citep{gu2023how} suggest initializing near $1/K$~($K$ being the sequence length) yields good results. 
Our theory gives grounding to this initialization practice, as well as to the S4D-Lin approximation, using a different viewpoint: our closed-form approximation for the filter $\delta_K$ in Eq.~\eqref{Param_of_the_as} is $a_s = \exp\left(-\frac{\alpha}{K}\right)\exp\left(i\frac{\pi s}{K}\right)$, with $s \in \llbracket -T, T\rrbracket$ and $S = 2T+1$. According to Lemma~\ref{Lemma param of bs}, for numerical stability $\alpha$ should be a small scalar. For $\alpha=1/2$, we get $\exp\left(-\frac{1}{2K}\right)\exp\left(i\frac{\pi s}{K}\right)$, i.e., exactly S4DLin with $\Delta =1/K$. We believe this connection to be a piece of evidence motivating correlation between magnitude and phase in modern variants of S4.

\section{Experiments}

We conclude our analysis of the copy task problem with numerical experiments illustrating the potential benefits of initialization of the parameters of linear models using representations from Eq.~\eqref{Param_of_the_as} and Eq.~\eqref{Param_of_the_bs}. We consider the following task: Given a dataset of autoregressive sequences $U = (u_1, u_2, \dots, u_N)$ of length $N$, generated as: \(u_n = \rho u_{n-1} + \varepsilon_n, \  \varepsilon_n \sim \mathcal{N}(0, \sigma^2), \ u_1 \sim \mathcal{U}(0,1),\) where $\rho \in [0,1)$ is the correlation factor and $\sigma^2 = 1 - \rho^2$, the task is for the model to restore the output \(Y = u_{t^*}\) for a fixed index $t^*$ in the sequence. This boils down to learning a shift of $K^* = N-t^*$ with a finite number of samples. We use an input-independent linear model as in Eq.~\eqref{eq:filter_new}, where the vector $a$ is initialized with Eq.~\eqref{Param_of_the_as}, and vector $b$ with Eq.~\eqref{Param_of_the_bs}. In an initial set of experiments, we demonstrate the advantages of initializing with linearly-spaced phases for tasks with a large horizon, compared to random initialization with phases sampled across the entire disk. Subsequently, we assess the robustness of gridded initialization to variations in $K_\text{init}$, highlighting its flexibility—a crucial property for real-world applications where the task horizon is typically uncertain. See results in Fig.~\ref{fig:xps}, where (left plot) we see that our filter yields increasing benefits as $\rho$ grows compared to initialization with random phases, and (right plot) the optimal performance is obtained with the correct $K_\text{init} = K^*$ for initialization, yet the method remains robust even when initialized far from the optimal value. In all experiments, we took $S=128$.
 \begin{figure}[h]
     \centering
     \begin{minipage}{0.55\linewidth}
         \centering
        \includegraphics[width=\linewidth]{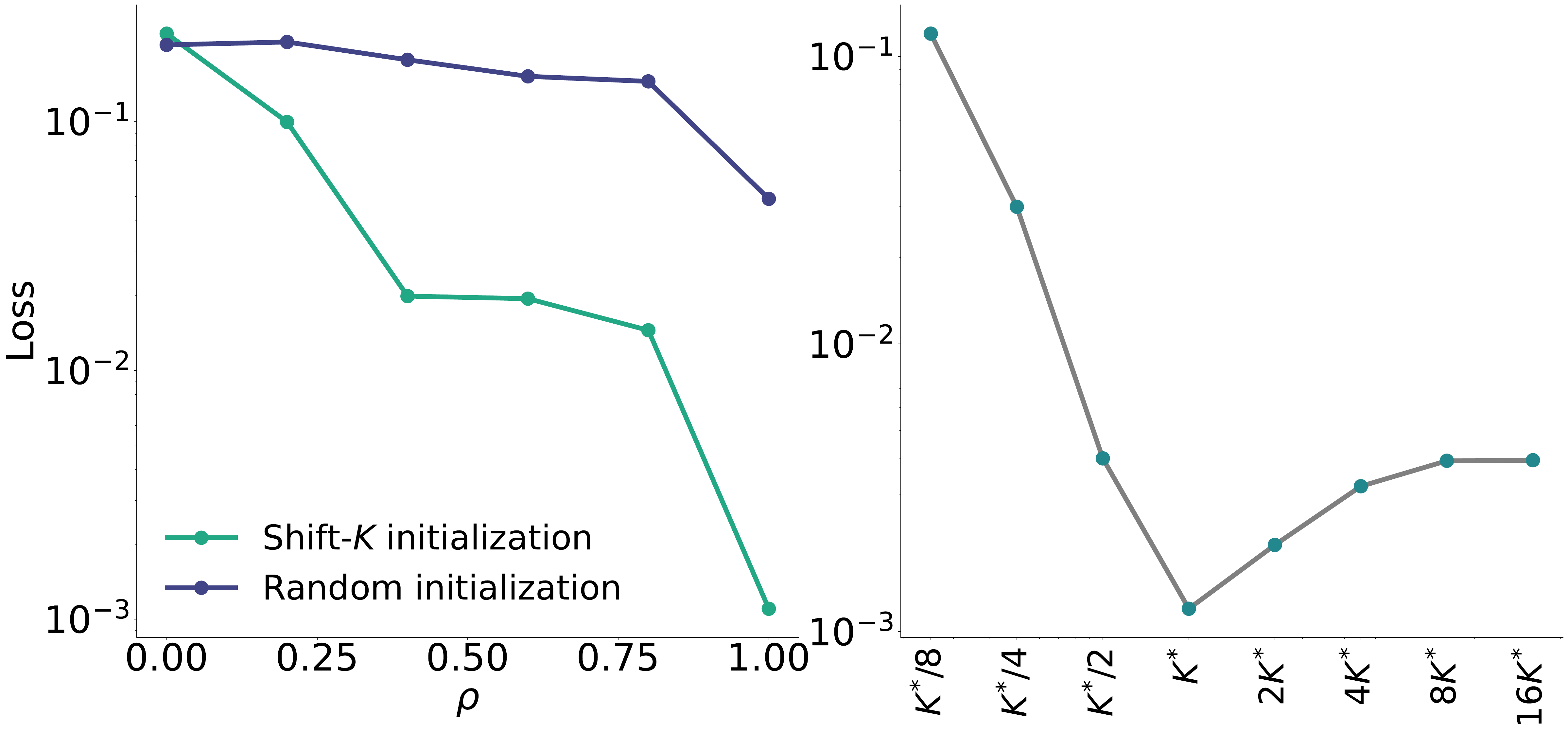}
     \end{minipage}
     \hspace{1cm}
     \begin{minipage}{0.35\linewidth}
         \vspace{0pt}
         \caption{{Initialization with regularly spaces phases enhances robustness and outperforms random initialization near the unit disk.}  
        {Left.} For $N=1500$ and $t^* = 200$, initialization using our filter defined in Eq.~\eqref{Param_of_the_as}  and Eq.~\eqref{Param_of_the_bs}.
        {Right.} For $N=2250$ and $t^*=250$, the task consists of learning a shift-$K$ filter with $K^*=2000$. Here, $\rho = 0.7$.  }
         \label{fig:xps}
         \vspace{0pt}
     \end{minipage}
 \end{figure}
\vspace{-5mm}
\section{Conclusion}\label{section discussion}

We demonstrated that the performance of linear models on a simplified copy task, when applied to stationary input sequences, depends on the ratio $\frac{S}{K}$, where $S$ denotes the size of the state space, and $K$ represents the lag of the copy task. This analysis revealed a form of uncertainty principle governing the resolution of our copy task with linear recurrences. To explain this trade-off between memory capacity and filter performance, we introduced a new filter that achieves the same performance on our copy task up to constants. This representation offers fresh insights into the filter's behavior, particularly in the spectral domain. As highlighted by \cite{orvieto2023resurrecting} and further elaborated by \cite{gu2022efficiently}, the initialization of the recurrence matrix's entries plays a crucial role in achieving high performance. Specifically, these studies constrain both the magnitudes and phases of the diagonal entries to depend on $\frac{1}{\Delta}$, where $\Delta$ has an order of magnitude similar to the sequence length. In this paper, we aim to provide an explanation for the efficacy of this specific initialization: it arises from the linear model's endeavor to retain certain elements of the sequence, thereby approximating the shifted Dirac function.

\section*{Acknowledgements}
The authors would like to thank Laurent Baratchart and Sylvain Chevillard for their helpful discussion on rational approximations of the complex exponential. We also thank Sajad Movahedi and Felix Sarnthein for their helpful comments on this manuscript.
This work has received support from the French government, managed by the National Research Agency, under the France 2030 program with the reference "PR[AI]RIE-PSAI" (ANR-23-IACL-0008). 
Antonio Orvieto is supported by the Hector Foundation.

\bibliography{main.bib}

\begin{thebibliography}{75}
\providecommand{\natexlab}[1]{#1}
\providecommand{\url}[1]{\texttt{#1}}
\expandafter\ifx\csname urlstyle\endcsname\relax
  \providecommand{\doi}[1]{doi: #1}\else
  \providecommand{\doi}{doi: \begingroup \urlstyle{rm}\Url}\fi

\bibitem[Ali et~al.(2024)Ali, Zimerman, and Wolf]{ali2024hidden}
Ameen Ali, Itamar Zimerman, and Lior Wolf.
\newblock The hidden attention of {Mamba} models.
\newblock \emph{arXiv preprint arXiv:2403.01590}, 2024.

\bibitem[Arora et~al.(2023)Arora, Eyuboglu, Timalsina, Johnson, Poli, Zou, Rudra, and Re]{arora2023zoology}
Simran Arora, Sabri Eyuboglu, Aman Timalsina, Isys Johnson, Michael Poli, James Zou, Atri Rudra, and Christopher Re.
\newblock Zoology: Measuring and improving recall in efficient language models.
\newblock In \emph{International Conference on Learning Representations}, 2023.

\bibitem[Baratchart et~al.(2016)Baratchart, Chevillard, and Qian]{baratchart2016minimax}
Laurent Baratchart, Sylvain Chevillard, and Tao Qian.
\newblock Minimax principle and lower bounds in {H2}-rational approximation.
\newblock \emph{Journal of Approximation Theory}, 206:\penalty0 17--47, 2016.

\bibitem[Barron(1993)]{barron1993universal}
Andrew~R. Barron.
\newblock Universal approximation bounds for superpositions of a sigmoidal function.
\newblock \emph{IEEE Transactions on Information Theory}, 39\penalty0 (3):\penalty0 930--945, 1993.

\bibitem[Bengio et~al.(1994)Bengio, Simard, and Frasconi]{bengio1994learning}
Yoshua Bengio, Patrice Simard, and Paolo Frasconi.
\newblock Learning long-term dependencies with gradient descent is difficult.
\newblock \emph{IEEE Transactions on Neural Networks}, 5\penalty0 (2):\penalty0 157--166, 1994.

\bibitem[Bhatia(2013)]{bhatia2013matrix}
Rajendra Bhatia.
\newblock \emph{Matrix Analysis}, volume 169.
\newblock Springer Science \& Business Media, 2013.

\bibitem[Boyd and Chua(1985)]{boyd1985fading}
Stephen Boyd and Leon Chua.
\newblock Fading memory and the problem of approximating nonlinear operators with volterra series.
\newblock \emph{IEEE Transactions on Circuits and Systems}, 32\penalty0 (11):\penalty0 1150--1161, 1985.

\bibitem[Brockwell and Davis(2002)]{brockwell2002introduction}
Peter~J. Brockwell and Richard~A. Davis.
\newblock \emph{Introduction to Time Series and Forecasting}.
\newblock Springer, 2002.

\bibitem[Brown et~al.(2020)Brown, Mann, Ryder, Subbiah, Kaplan, Dhariwal, Neelakantan, Shyam, Sastry, Askell, et~al.]{brown2020language}
Tom Brown, Benjamin Mann, Nick Ryder, Melanie Subbiah, Jared~D Kaplan, Prafulla Dhariwal, Arvind Neelakantan, Pranav Shyam, Girish Sastry, Amanda Askell, et~al.
\newblock Language models are few-shot learners.
\newblock \emph{Advances in Neural Information Processing Systems}, 2020.

\bibitem[Calvetti and Reichel(1996)]{calvetti1996solution}
D.~Calvetti and L.~Reichel.
\newblock On the solution of {C}auchy systems of equations.
\newblock \emph{Electronic Transactions on Numerical Analysis}, 4:\penalty0 125--137, 1996.

\bibitem[Chen et~al.(2016)Chen, Xu, Zhang, and Guestrin]{chen2016training}
Tianqi Chen, Bing Xu, Chiyuan Zhang, and Carlos Guestrin.
\newblock Training deep nets with sublinear memory cost.
\newblock \emph{arXiv preprint arXiv:1604.06174}, 2016.

\bibitem[Chen et~al.(2021)Chen, Zeng, Ji, and Yang]{chen2021skyformer}
Yifan Chen, Qi~Zeng, Heng Ji, and Yun Yang.
\newblock Skyformer: Remodel self-attention with {G}aussian kernel and {N}ystrom method.
\newblock \emph{Advances in Neural Information Processing Systems}, 2021.

\bibitem[Choromanski et~al.(2020)Choromanski, Likhosherstov, Dohan, Song, Gane, Sarlos, Hawkins, Davis, Mohiuddin, Kaiser, et~al.]{choromanski2020rethinking}
Krzysztof~Marcin Choromanski, Valerii Likhosherstov, David Dohan, Xingyou Song, Andreea Gane, Tamas Sarlos, Peter Hawkins, Jared~Quincy Davis, Afroz Mohiuddin, Lukasz Kaiser, et~al.
\newblock Rethinking attention with performers.
\newblock In \emph{International Conference on Learning Representations}, 2020.

\bibitem[Cirone et~al.(2024)Cirone, Orvieto, Walker, Salvi, and Lyons]{cirone2024theoretical}
Nicola~Muca Cirone, Antonio Orvieto, Benjamin Walker, Cristopher Salvi, and Terry Lyons.
\newblock Theoretical foundations of deep selective state-space models.
\newblock In \emph{Advances in Neural Information Processing Systems}, 2024.

\bibitem[Dalla-Torre et~al.(2024)Dalla-Torre, Gonzalez, Mendoza-Revilla, Lopez~Carranza, Grzywaczewski, Oteri, Dallago, Trop, de~Almeida, Sirelkhatim, et~al.]{dalla2024nucleotide}
Hugo Dalla-Torre, Liam Gonzalez, Javier Mendoza-Revilla, Nicolas Lopez~Carranza, Adam~Henryk Grzywaczewski, Francesco Oteri, Christian Dallago, Evan Trop, Bernardo~P. de~Almeida, Hassan Sirelkhatim, et~al.
\newblock Nucleotide transformer: building and evaluating robust foundation models for human genomics.
\newblock \emph{Nature Methods}, 21\penalty0 (11):\penalty0 1--11, 2024.

\bibitem[Dao(2023)]{dao2023flashattention}
Tri Dao.
\newblock Flashattention-2: Faster attention with better parallelism and work partitioning.
\newblock \emph{arXiv preprint arXiv:2307.08691}, 2023.

\bibitem[Dao and Gu(2024)]{dao2024transformers}
Tri Dao and Albert Gu.
\newblock Transformers are {SSM}s: Generalized models and efficient algorithms through structured state space duality.
\newblock In \emph{International Conference on Machine Learning}, 2024.

\bibitem[Dao et~al.(2022)Dao, Fu, Ermon, Rudra, and R{\'e}]{dao2022flashattention}
Tri Dao, Dan Fu, Stefano Ermon, Atri Rudra, and Christopher R{\'e}.
\newblock Flashattention: Fast and memory-efficient exact attention with io-awareness.
\newblock \emph{Advances in Neural Information Processing Systems}, 2022.

\bibitem[De et~al.(2024)De, Smith, Fernando, Botev, Cristian-Muraru, Gu, Haroun, Berrada, Chen, Srinivasan, et~al.]{de2024griffin}
Soham De, Samuel~L. Smith, Anushan Fernando, Aleksandar Botev, George Cristian-Muraru, Albert Gu, Ruba Haroun, Leonard Berrada, Yutian Chen, Srivatsan Srinivasan, et~al.
\newblock Griffin: Mixing gated linear recurrences with local attention for efficient language models.
\newblock \emph{arXiv preprint arXiv:2402.19427}, 2024.

\bibitem[Dosovitskiy et~al.(2021)Dosovitskiy, Beyer, Kolesnikov, Weissenborn, Zhai, Unterthiner, Dehghani, Minderer, Heigold, Gelly, Uszkoreit, and Houlsby]{dosovitskiy2021image}
Alexey Dosovitskiy, Lucas Beyer, Alexander Kolesnikov, Dirk Weissenborn, Xiaohua Zhai, Thomas Unterthiner, Mostafa Dehghani, Matthias Minderer, Georg Heigold, Sylvain Gelly, Jakob Uszkoreit, and Neil Houlsby.
\newblock An image is worth 16x16 words: Transformers for image recognition at scale.
\newblock In \emph{International Conference on Learning Representations}, 2021.

\bibitem[Elman(1990)]{elman1990finding}
Jeffrey~L. Elman.
\newblock Finding structure in time.
\newblock \emph{Cognitive Science}, 14\penalty0 (2):\penalty0 179--211, 1990.

\bibitem[Goel et~al.(2022)Goel, Gu, Donahue, and R{\'e}]{goel2022sashimi}
Karan Goel, Albert Gu, Chris Donahue, and Christopher R{\'e}.
\newblock It's raw! audio generation with state-space models.
\newblock \emph{International Conference on Machine Learning}, 2022.

\bibitem[Gray(2006)]{gray2006toeplitz}
Robert~M. Gray.
\newblock Toeplitz and circulant matrices: A review.
\newblock \emph{Foundations and Trends in Communications and Information Theory}, 2\penalty0 (3):\penalty0 155--239, 2006.

\bibitem[Gu and Dao(2024)]{gu2024mamba}
Albert Gu and Tri Dao.
\newblock Mamba: Linear-time sequence modeling with selective state spaces.
\newblock In \emph{Conference on Language Modeling}, 2024.

\bibitem[Gu et~al.(2020)Gu, Dao, Ermon, Rudra, and R{\'e}]{gu2020hippo}
Albert Gu, Tri Dao, Stefano Ermon, Atri Rudra, and Christopher R{\'e}.
\newblock Hippo: Recurrent memory with optimal polynomial projections.
\newblock In \emph{Advances in Neural Information Processing Systems}, 2020.

\bibitem[Gu et~al.(2022{\natexlab{a}})Gu, Goel, Gupta, and R{\'e}]{gu2022parameterization}
Albert Gu, Karan Goel, Ankit Gupta, and Christopher R{\'e}.
\newblock On the parameterization and initialization of diagonal state space models.
\newblock In \emph{Advances in Neural Information Processing Systems}, 2022{\natexlab{a}}.

\bibitem[Gu et~al.(2022{\natexlab{b}})Gu, Goel, and Re]{gu2022efficiently}
Albert Gu, Karan Goel, and Christopher Re.
\newblock Efficiently modeling long sequences with structured state spaces.
\newblock In \emph{International Conference on Learning Representations}, 2022{\natexlab{b}}.

\bibitem[Gu et~al.(2023)Gu, Johnson, Timalsina, Rudra, and Re]{gu2023how}
Albert Gu, Isys Johnson, Aman Timalsina, Atri Rudra, and Christopher Re.
\newblock How to train your {HIPPO}: State space models with generalized orthogonal basis projections.
\newblock In \emph{International Conference on Learning Representations}, 2023.

\bibitem[Gupta et~al.(2022)Gupta, Gu, and Berant]{gupta2022diagonal}
Ankit Gupta, Albert Gu, and Jonathan Berant.
\newblock Diagonal state spaces are as effective as structured state spaces.
\newblock In \emph{Advances in Neural Information Processing Systems}, 2022.

\bibitem[Hanson and Raginsky(2019)]{hanson2019universal}
Joshua Hanson and Maxim Raginsky.
\newblock Universal approximation of input-output maps by temporal convolutional nets.
\newblock \emph{Advances in Neural Information Processing Systems}, 2019.

\bibitem[Hanson and Raginsky(2020)]{hanson2020universal}
Joshua Hanson and Maxim Raginsky.
\newblock Universal simulation of stable dynamical systems by recurrent neural nets.
\newblock In \emph{Learning for Dynamics and Control}, 2020.

\bibitem[Hespanha(2018)]{hespanha2018linear}
Joao~P. Hespanha.
\newblock \emph{Linear Systems Theory}.
\newblock Princeton University Press, 2018.

\bibitem[Hochreiter and Schmidhuber(1997)]{hochreiter1997long}
Sepp Hochreiter and J{\"u}rgen Schmidhuber.
\newblock Long short-term memory.
\newblock \emph{Neural Computation}, 9\penalty0 (8):\penalty0 1735--1780, 1997.

\bibitem[Jaeger(2001)]{jaeger2001echo}
Herbert Jaeger.
\newblock The "echo state" approach to analysing and training recurrent neural networks-with an erratum note.
\newblock \emph{German National Research Center for Information Technology GMD Technical Report}, 2001.

\bibitem[Jelassi et~al.(2024)Jelassi, Brandfonbrener, and Kakade]{jelassi2024repeat}
Samy Jelassi, David Brandfonbrener, and Sham~M. Kakade.
\newblock Repeat after me: Transformers are better than state space models at copying.
\newblock In \emph{International Conference on Machine Learning}, 2024.

\bibitem[Katharopoulos et~al.(2020)Katharopoulos, Vyas, Pappas, and Fleuret]{katharopoulos2020transformers}
Angelos Katharopoulos, Apoorv Vyas, Nikolaos Pappas, and Fran{\c{c}}ois Fleuret.
\newblock Transformers are {RNN}s: Fast autoregressive transformers with linear attention.
\newblock In \emph{International Conference on Machine Learning}, 2020.

\bibitem[Korsky(2019)]{korsky2019computational}
Samuel~A. Korsky.
\newblock \emph{On the Computational Power of RNNs}.
\newblock PhD thesis, Massachusetts Institute of Technology, 2019.

\bibitem[Lee-Thorp et~al.(2022)Lee-Thorp, Ainslie, Eckstein, and Ontanon]{lee2022fnet}
James Lee-Thorp, Joshua Ainslie, Ilya Eckstein, and Santiago Ontanon.
\newblock Fnet: Mixing tokens with {Fourier} transforms.
\newblock In \emph{North American Chapter of the Association for Computational Linguistics: Human Language Technologies}, 2022.

\bibitem[Li et~al.(2025)Li, Li, Wang, He, Wang, Wang, and Qiao]{li2025videomamba}
Kunchang Li, Xinhao Li, Yi~Wang, Yinan He, Yali Wang, Limin Wang, and Yu~Qiao.
\newblock Videomamba: State space model for efficient video understanding.
\newblock In \emph{European Conference on Computer Vision}, 2025.

\bibitem[Li et~al.(2022)Li, Han, Weinan, and Li]{li2022approximation}
Zhong Li, Jiequn Han, E~Weinan, and Qianxiao Li.
\newblock Approximation and optimization theory for linear continuous-time recurrent neural networks.
\newblock \emph{Journal of Machine Learning Research}, 23\penalty0 (42):\penalty0 1--85, 2022.

\bibitem[Liang et~al.(2024)Liang, Zhou, Xu, Zhu, Zou, Ye, Tan, and Bai]{liang2024pointmamba}
Dingkang Liang, Xin Zhou, Wei Xu, Xingkui Zhu, Zhikang Zou, Xiaoqing Ye, Xiao Tan, and Xiang Bai.
\newblock Point{M}amba: A simple state space model for point cloud analysis.
\newblock In \emph{Advances in Neural Information Processing Systems}, 2024.

\bibitem[Liu et~al.(2024)Liu, Tian, Zhao, Yu, Xie, Wang, Ye, Jiao, and Liu]{liu2024vmamba}
Yue Liu, Yunjie Tian, Yuzhong Zhao, Hongtian Yu, Lingxi Xie, Yaowei Wang, Qixiang Ye, Jianbin Jiao, and Yunfan Liu.
\newblock {VM}amba: Visual state space model.
\newblock In \emph{Advances in Neural Information Processing Systems}, 2024.

\bibitem[Lu et~al.(2023)Lu, Schroecker, Gu, Parisotto, Foerster, Singh, and Behbahani]{lu2023structured}
Chris Lu, Yannick Schroecker, Albert Gu, Emilio Parisotto, Jakob Foerster, Satinder Singh, and Feryal Behbahani.
\newblock Structured state space models for in-context reinforcement learning.
\newblock In \emph{Advances in Neural Information Processing Systems}, 2023.

\bibitem[Ma et~al.(2023)Ma, Lin, Lim, Romero-Soriano, Dokania, Coates, Torr, and Lim]{ma2023graph}
Liheng Ma, Chen Lin, Derek Lim, Adriana Romero-Soriano, Puneet~K Dokania, Mark Coates, Philip Torr, and Ser-Nam Lim.
\newblock Graph inductive biases in transformers without message passing.
\newblock In \emph{International Conference on Machine Learning}, 2023.

\bibitem[Martin and Cundy(2018)]{martin2018parallelizing}
Eric Martin and Chris Cundy.
\newblock Parallelizing linear recurrent neural nets over sequence length.
\newblock In \emph{International Conference on Learning Representations}, 2018.

\bibitem[Merrill et~al.(2024)Merrill, Petty, and Sabharwal]{merrill2024illusion}
William Merrill, Jackson Petty, and Ashish Sabharwal.
\newblock The illusion of state in state-space models.
\newblock In \emph{International Conference on Machine Learning}, 2024.

\bibitem[Nguyen et~al.(2024)Nguyen, Poli, Durrant, Kang, Katrekar, Li, Bartie, Thomas, King, Brixi, et~al.]{nguyen2024sequence}
Eric Nguyen, Michael Poli, Matthew~G. Durrant, Brian Kang, Dhruva Katrekar, David~B. Li, Liam~J. Bartie, Armin~W. Thomas, Samuel~H. King, Garyk Brixi, et~al.
\newblock Sequence modeling and design from molecular to genome scale with evo.
\newblock \emph{Science}, 386\penalty0 (6723), 2024.

\bibitem[Olsson et~al.(2022)Olsson, Elhage, Nanda, Joseph, DasSarma, Henighan, Mann, Askell, Bai, Chen, et~al.]{olsson2022context}
Catherine Olsson, Nelson Elhage, Neel Nanda, Nicholas Joseph, Nova DasSarma, Tom Henighan, Ben Mann, Amanda Askell, Yuntao Bai, Anna Chen, et~al.
\newblock In-context learning and induction heads.
\newblock \emph{arXiv preprint arXiv:2209.11895}, 2022.

\bibitem[Oppenheim et~al.(1996)Oppenheim, Willsky, and Nawab]{oppenheim1996signals}
Alan~V. Oppenheim, Alan~S. Willsky, and S.~Hamid Nawab.
\newblock \emph{Signals \& Systems (2nd ed.)}.
\newblock Prentice-Hall, 1996.

\bibitem[Orvieto et~al.(2023)Orvieto, Smith, Gu, Fernando, Gulcehre, Pascanu, and De]{orvieto2023resurrecting}
Antonio Orvieto, Samuel~L. Smith, Albert Gu, Anushan Fernando, Caglar Gulcehre, Razvan Pascanu, and Soham De.
\newblock Resurrecting recurrent neural networks for long sequences.
\newblock In \emph{International Conference on Machine Learning}, 2023.

\bibitem[Orvieto et~al.(2024)Orvieto, De, Gulcehre, Pascanu, and Smith]{orvieto2024universality}
Antonio Orvieto, Soham De, Caglar Gulcehre, Razvan Pascanu, and Samuel~L. Smith.
\newblock Universality of linear recurrences followed by non-linear projections: Finite-width guarantees and benefits of complex eigenvalues.
\newblock In \emph{International Conference on Machine Learning}, 2024.

\bibitem[Pagnoni et~al.(2024)Pagnoni, Pasunuru, Rodriguez, Nguyen, Muller, Li, Zhou, Yu, Weston, Zettlemoyer, et~al.]{pagnoni2024byte}
Artidoro Pagnoni, Ram Pasunuru, Pedro Rodriguez, John Nguyen, Benjamin Muller, Margaret Li, Chunting Zhou, Lili Yu, Jason Weston, Luke Zettlemoyer, et~al.
\newblock Byte latent transformer: Patches scale better than tokens.
\newblock \emph{arXiv preprint arXiv:2412.09871}, 2024.

\bibitem[Pascanu et~al.(2013)Pascanu, Mikolov, and Bengio]{pascanu2013difficulty}
Razvan Pascanu, Tomas Mikolov, and Yoshua Bengio.
\newblock On the difficulty of training recurrent neural networks.
\newblock In \emph{International Conference on Machine Learning}, 2013.

\bibitem[Peng et~al.(2024)Peng, Goldstein, Anthony, Albalak, Alcaide, Biderman, Cheah, Du, Ferdinan, Hou, et~al.]{peng2024eagle}
Bo~Peng, Daniel Goldstein, Quentin Anthony, Alon Albalak, Eric Alcaide, Stella Biderman, Eugene Cheah, Xingjian Du, Teddy Ferdinan, Haowen Hou, et~al.
\newblock Eagle and {Finch}: {RWKV} with matrix-valued states and dynamic recurrence.
\newblock \emph{arXiv preprint arXiv:2404.05892}, 2024.

\bibitem[Qin et~al.(2024)Qin, Yang, Sun, Shen, Li, Sun, and Zhong]{qin2024hgrn2}
Zhen Qin, Songlin Yang, Weixuan Sun, Xuyang Shen, Dong Li, Weigao Sun, and Yiran Zhong.
\newblock {HGRN2}: Gated linear {RNN}s with state expansion.
\newblock \emph{arXiv preprint arXiv:2404.07904}, 2024.

\bibitem[Ran-Milo et~al.(2024)Ran-Milo, Lumbroso, Cohen-Karlik, Giryes, Globerson, and Cohen]{ran2024provable}
Yuval Ran-Milo, Eden Lumbroso, Edo Cohen-Karlik, Raja Giryes, Amir Globerson, and Nadav Cohen.
\newblock Provable benefits of complex parameterizations for structured state space models.
\newblock \emph{arXiv preprint arXiv:2410.14067}, 2024.

\bibitem[Rumelhart et~al.(1986)Rumelhart, Smolensky, McClelland, and Hinton]{rumelhart1986sequential}
David~E. Rumelhart, Paul Smolensky, James~L. McClelland, and G.~Hinton.
\newblock Sequential thought processes in pdp models.
\newblock \emph{Parallel Distributed Processing: Explorations in the Microstructures of Cognition}, 2:\penalty0 3--57, 1986.

\bibitem[Schlag et~al.(2021)Schlag, Irie, and Schmidhuber]{schlag2021linear}
Imanol Schlag, Kazuki Irie, and J{\"u}rgen Schmidhuber.
\newblock Linear transformers are secretly fast weight programmers.
\newblock In \emph{International Conference on Machine Learning}, 2021.

\bibitem[Serov(2017)]{serov2017fourier}
Valery Serov.
\newblock \emph{Fourier series, Fourier Transform and their Applications to Mathematical Physics}, volume 197.
\newblock Springer, 2017.

\bibitem[Sieber et~al.(2024)Sieber, Alonso, Didier, Zeilinger, and Orvieto]{sieber2024understanding}
Jerome Sieber, Carmen~Amo Alonso, Alexandre Didier, Melanie Zeilinger, and Antonio Orvieto.
\newblock Understanding the differences in foundation models: Attention, state space models, and recurrent neural networks.
\newblock In \emph{Advances in Neural Information Processing Systems}, 2024.

\bibitem[Siegelmann and Sontag(1992)]{siegelmann1992computational}
Hava~T. Siegelmann and Eduardo~D. Sontag.
\newblock On the computational power of neural nets.
\newblock In \emph{Proceedings of the fifth Annual Workshop on Computational Learning Theory}, 1992.

\bibitem[Smith et~al.(2023)Smith, Warrington, and Linderman]{smith2023simplified}
Jimmy~T.H. Smith, Andrew Warrington, and Scott Linderman.
\newblock Simplified state space layers for sequence modeling.
\newblock In \emph{International Conference on Learning Representations}, 2023.

\bibitem[Tay et~al.(2020)Tay, Dehghani, Abnar, Shen, Bahri, Pham, Rao, Yang, Ruder, and Metzler]{tay2020long}
Yi~Tay, Mostafa Dehghani, Samira Abnar, Yikang Shen, Dara Bahri, Philip Pham, Jinfeng Rao, Liu Yang, Sebastian Ruder, and Donald Metzler.
\newblock Long range arena: A benchmark for efficient transformers.
\newblock In \emph{International Conference on Learning Representations}, 2020.

\bibitem[Team et~al.(2024)Team, Georgiev, Lei, Burnell, Bai, Gulati, Tanzer, Vincent, Pan, Wang, et~al.]{team2024gemini}
Gemini Team, Petko Georgiev, Ving~Ian Lei, Ryan Burnell, Libin Bai, Anmol Gulati, Garrett Tanzer, Damien Vincent, Zhufeng Pan, Shibo Wang, et~al.
\newblock Gemini 1.5: Unlocking multimodal understanding across millions of tokens of context.
\newblock \emph{arXiv preprint arXiv:2403.05530}, 2024.

\bibitem[Trockman et~al.(2024)Trockman, Harutyunyan, Kolter, Kumar, and Bhojanapalli]{trockman2024mimetic}
Asher Trockman, Hrayr Harutyunyan, J~Zico Kolter, Sanjiv Kumar, and Srinadh Bhojanapalli.
\newblock Mimetic initialization helps state space models learn to recall.
\newblock \emph{arXiv preprint arXiv:2410.11135}, 2024.

\bibitem[Vaswani et~al.(2017)Vaswani, Shazeer, Parmar, Uszkoreit, Jones, Gomez, Kaiser, and Polosukhin]{vaswani2017attention}
Ashish Vaswani, Noam Shazeer, Niki Parmar, Jakob Uszkoreit, Llion Jones, Aidan~N Gomez, Lukasz Kaiser, and Illia Polosukhin.
\newblock Attention is all you need.
\newblock \emph{Advances in Neural Information Processing Systems}, 2017.

\bibitem[Waleffe et~al.(2024)Waleffe, Byeon, Riach, Norick, Korthikanti, Dao, Gu, Hatamizadeh, Singh, Narayanan, et~al.]{waleffe2024empirical}
Roger Waleffe, Wonmin Byeon, Duncan Riach, Brandon Norick, Vijay Korthikanti, Tri Dao, Albert Gu, Ali Hatamizadeh, Sudhakar Singh, Deepak Narayanan, et~al.
\newblock An empirical study of {M}amba-based language models.
\newblock \emph{arXiv preprint arXiv:2406.07887}, 2024.

\bibitem[Wang and Xue(2024)]{wang2024state}
Shida Wang and Beichen Xue.
\newblock State-space models with layer-wise nonlinearity are universal approximators with exponential decaying memory.
\newblock In \emph{Advances in Neural Information Processing Systems}, 2024.

\bibitem[Wang et~al.(2020)Wang, Li, Khabsa, Fang, and Ma]{wang2020linformer}
Sinong Wang, Belinda~Z. Li, Madian Khabsa, Han Fang, and Hao Ma.
\newblock Linformer: Self-attention with linear complexity.
\newblock \emph{arXiv preprint arXiv:2006.04768}, 2020.

\bibitem[Xing et~al.(2024)Xing, Ye, Yang, Liu, and Zhu]{xing2024segmamba}
Zhaohu Xing, Tian Ye, Yijun Yang, Guang Liu, and Lei Zhu.
\newblock Segmamba: Long-range sequential modeling {Mamba} for 3{D} medical image segmentation.
\newblock In \emph{International Conference on Medical Image Computing and Computer-Assisted Intervention}, 2024.

\bibitem[Yang et~al.(2024{\natexlab{a}})Yang, Wang, Shen, Panda, and Kim]{yang2024gated}
Songlin Yang, Bailin Wang, Yikang Shen, Rameswar Panda, and Yoon Kim.
\newblock Gated linear attention transformers with hardware-efficient training.
\newblock In \emph{International Conference on Machine Learning}, 2024{\natexlab{a}}.

\bibitem[Yang et~al.(2024{\natexlab{b}})Yang, Wang, Zhang, Shen, and Kim]{yang2024parallelizing}
Songlin Yang, Bailin Wang, Yu~Zhang, Yikang Shen, and Yoon Kim.
\newblock Parallelizing linear transformers with the delta rule over sequence length.
\newblock \emph{arXiv preprint arXiv:2406.06484}, 2024{\natexlab{b}}.

\bibitem[Yang(2003)]{yang2003generalized}
Zheng-Hong Yang.
\newblock Generalized confluent {C}auchy--{V}andermonde matrices: displacement structures, inversion formulas and tangential interpolations.
\newblock \emph{Journal of Computational and Applied Mathematics}, 154\penalty0 (2):\penalty0 355--371, 2003.

\bibitem[Zucchet and Orvieto(2024)]{zucchet2024recurrent}
Nicolas Zucchet and Antonio Orvieto.
\newblock Recurrent neural networks: vanishing and exploding gradients are not the end of the story.
\newblock In \emph{Advances in Neural Information Processing Systems}, 2024.

\bibitem[Zucchet et~al.(2023)Zucchet, Meier, Schug, Mujika, and Sacramento]{zucchet2023online}
Nicolas Zucchet, Robert Meier, Simon Schug, Asier Mujika, and Jo{\~a}o Sacramento.
\newblock Online learning of long-range dependencies.
\newblock In \emph{Advances in Neural Information Processing Systems}, 2023.

\end{thebibliography}

\newpage
\setcounter{section}{0}

\begin{appendices}

In this Appendix, we provide a detailed proof for all our theoretical results. We start in Appendix~\ref{RNN basics} with an equivalence of various representations of linear RNNs, then in Appendix~\ref{review} with a review of fundamentals of signal processing.
\listofappendices

\counterwithin{figure}{section}
\counterwithin{table}{section}

\newpage

\section{Recurrent Neural Networks and Diagonal forms}

\label{RNN basics}

Linear recurrent networks such as SSMs, in their simplest form, are causal models acting on a $d$ dimensional input sequence with $L$ elements $U\in\mathbb{R}^{d\times L}$, producing an output sequence $Y\in\mathbb{R}^{d\times L}$ through a filtering process parametrized by variables $A\in\mathbb{R}^{N\times N}$, $B\in\mathbb{R}^{N\times d}, P\in\mathbb{R}^{d\times N}$. Let $U_n\in\mathbb{R}^{d}$ denote the $n$-th timestamp data contained in $U$, a linear RNN processes the inputs as follows~\citep{gu2022parameterization,orvieto2023resurrecting}
\begin{equation}
    X_{n} = A X_{n-1} + B U_n,\qquad Y_{n} = PX_n.
    \label{eq:appendix linear_RNN}
\end{equation}

\begin{proposition}[Linear RNNs and convolution form]
    Let $A\in\mathbb{R}^{S\times S}$ such that $A$ is diagonal, $B\in\mathbb{R}^{S\times 1}, P\in\mathbb{R}^{1\times S}$, and $u = (u_n)_{n\in\mathbb{Z}}$ be a univariate input signal. The output signal $(y_n)_{n\in\mathbb{Z}}$ can write 
    \[
    y_n = \sum_{k=0}^\infty c_ku_{n-k}
    \]
    with $c_k = \sum_{s=1}^Sa_s^kb_s$.
\end{proposition}

\begin{proof}
    We have $A = \begin{pmatrix}
        a_1 & \dots & \\
        & \ddots & \\
        & \dots & a_S
    \end{pmatrix}, B = \begin{pmatrix}
        b_1\\
        \vdots\\
        b_S
    \end{pmatrix}, P = \begin{pmatrix}
        p_1 & \dots & p_S
    \end{pmatrix}$.

\begin{align*}
    X_n &= AX_{n-1} + Bu_n\\
    &= A(AX_{n-2} + Bu_{n-1}) + Bu_n  = \dots = \sum_{k=0}^nA^kBu_{n-k}\\
    &= \sum_{k=0}^n\begin{pmatrix}
        a_1^k & \dots & \\
        & \ddots & \\
        & \dots & a_S^k
    \end{pmatrix}\begin{pmatrix}
        b_1\\\vdots\\b_S
    \end{pmatrix}u_{n-k} = \sum_{k=0}^n\begin{pmatrix}
        a_1^kb_1\\\vdots \\a_S^kb_s
    \end{pmatrix}u_{n-k}.
\end{align*}
Finally,
\begin{align*}
    y_n = \begin{pmatrix}
        p_1 & \dots & p_S
    \end{pmatrix}X_n = \sum_{k=0}^n\begin{pmatrix}
        p_1 & \dots & p_S
    \end{pmatrix}\begin{pmatrix}
        a_1^kb_1\\\vdots\\a_N^kb_N
    \end{pmatrix}u_{n-k} & = \sum_{s=1}^S\sum_{k=0}^np_sa_s^kb_su_{n-k}\\
    &=\sum_{k=0}^nu_{n-k}\sum_{s=1}^Sa_s^kb_sp_s = \sum_{k=0}^nu_{n-k}c_k,
\end{align*}
with $c_k = \sum_{s=1}^Sa_s^kb_sp_s$. In this paper, we consider without loss of generality $\begin{pmatrix}
p_1 \dots p_s
\end{pmatrix} = \begin{pmatrix}
    1 \dots 1
\end{pmatrix}.$
\end{proof}
\section{Some fundamentals of signal processing}\label{section fundamentals} 
\label{review}

In this section, we will recall some fundamentals definitions and results in signal processing. We will only look at discrete-time signals. Throughout this section, we denote $(x_n)_{n\in\mathbb{Z}}$ or $x_n$ a discrete time signal, and $x_k$ the value taken by the signal at time $k$. For example, let us denote $(e_n)$ the impulse signal such that 
\begin{equation}
e_n =
\begin{cases}
    1, n = 0\\
    0, n\neq 0.
\end{cases}  
\label{appendix impulse signal}
\end{equation}
This signal is  useful because the response of a system to a impulse signal gives a lot of insights. In particular it fully describes a linear time-invariant system. For more on signal processing, we refer the reader to \cite{oppenheim1996signals}.

\subsection{Linear Time-invariant systems}

A system is said to be \textit{time-invariant} if its response to a certain input signal does not depend on time. It is said to be \textit{linear} if its output response to a linear combinations of inputs is the same linear combinations of the output responses of the individual inputs. A system is said to be \textit{causal} if the output at a present time depends on the input up the present time only. 

There exist several ways to represent the input-output behavior of LTI system. We will only look at the impulse response representation (convolution).

\begin{proposition}[Convolution]
    Let $h_n$ be the impulse response of an LTI system $H$ (i.e., the output of system $H$ subject to input $e_n$), and $x_n$ be an input signal. In this case, the output signal of the system $y_n$ writes 
    \begin{equation}
        y_n = \sum_{k=-\infty}^{+\infty}x_kh_{n-k}.
        \label{appendix conv LTI}
    \end{equation}
\end{proposition}

\textit{Causal systems.} The output $y_n$ of a causal system depends only on past or present values of the input. This forces $h_k=0$ for $k<0$ and the convolution sum is rewritten 
\[
y_n = \sum_{k=0}^{+\infty}h_kx_{n-k}.
\]

\textit{Stable systems.} A system is stable if the output is guaranteed to be bounded for every bounded input.

\subsection{Discrete-Time Fourier Transform}

In this section, we denote $x_n$ a complex-valued discrete-time signal.

\begin{definition}
    The discrete-time Fourier transform of signal $x_n$ is given by
    \[
    X(\omega) = \sum_{n=-\infty}^{+\infty}x_ne^{-i\omega n}.
    \]
    This function takes values in the frequency space.
    The inverse discrete-time Fourier transform is given by 
    \[
    x_n = \frac{1}{2\pi}\int_0^{2\pi}X(\omega)e^{i\omega n}d\omega.
    \]
\end{definition}

The Discrete-Time Fourier transform presents some notable properties that we recall in Table~\ref{table:dtft-properties}.

\begin{table}[h!]
\centering
\renewcommand{\arraystretch}{1.5}
\begin{tabular}{|c|c|}
\hline
\textbf{Property} & \textbf{Relation} \\ \hline
Time Shifting & 
$x_{n-k} \overset{DTFT}{\longleftrightarrow} e^{-i\omega k} X(\omega)$ \\ \hline
Convolution in Time & 
$x_n * y_n \overset{DTFT}{\longleftrightarrow} X(\omega) Y(\omega)$ \\ \hline
Frequency Differentiation & 
$j \frac{d}{d\omega} X(\omega) \overset{DTFT}{\longleftrightarrow} -n x_n$ \\ \hline
Differencing in Time & 
$x_n - x{n-1} \overset{DTFT}{\longleftrightarrow} \left(1 - e^{-i\omega}\right) X(\omega)$ \\ \hline
\end{tabular}
\caption{Properties of the Discrete-Time Fourier Transform (DTFT). For each property, assume $x_n\overset{DTFT}{\longleftrightarrow} X(\omega)$ and $y_n\overset{DTFT}{\longleftrightarrow} Y(\omega)$.}
\label{table:dtft-properties}
\end{table}

We recall Parseval's theorem that establishes a fundamental equivalence between the inner product of two signals in the time domain and their corresponding representation in the frequency domain.

\begin{theorem}[Parseval]
    For two complex-valued discrete-time signals \((x_n)\) and \((y_n)\) with discrete-time Fourier transforms \(X(e^{i\omega})\) and \(Y(e^{i\omega})\), Parseval's theorem yields:
    \begin{equation}
        \sum_{n=-\infty}^{+\infty}x_n\overline{y_n} = \frac{1}{2\pi}\int_0^{2\pi}X(e^{i\omega})\overline{Y(e^{i\omega})}d\omega.
        \label{Parseval thm}
    \end{equation}
In particular, Parseval's theorem yields an energy conservation result:
    $$
        \sum_{n=-\infty}^{+\infty}\vert x_n\vert^2 = \frac{1}{2\pi}\int_0^{2\pi}\vert X(e^{i\omega})\vert^2 d\omega.
    $$
\end{theorem}
The following proposition will be useful in our lower bound proof in Appendix~\ref{appendix subsection white noise loss}.
\begin{proposition}\label{proposition semi parseval}
    Let $w_n$ be a causal discrete-time complex-valued signal with Fourier transform $W(\omega)$. We have the following equality:
    \[
    \sum_{L=0}^{+\infty}L\vert w_l\vert^2 = \frac{i}{2\pi}\int_0^{2\pi}\frac{dW(\omega)}{d\omega}\overline{W}(\omega)d\omega.
    \]
\end{proposition}

\begin{proof}
    By definition of the DTFT, $W(\omega) = \sum_{L=0}^{+\infty}w_Le^{-i\omega L}$. Therefore, 
    \begin{align*}
        \sum_{L=0}^{+\infty}L\vert w_L\vert^2 &= \sum_{L=0}^{+\infty}Lw_L\bar{w}_L = \frac{1}{2\pi}\sum_{L=0}^{+\infty}\sum_{L'=0}^{+\infty}Lw_L\bar{w}_{L'}\int_0^{2\pi}e^{-i\omega(L-L')}d\omega\\
        &= \frac{i}{2\pi}\int_0^{2\pi}\sum_{L=0}^{+\infty}-iL\omega_Le^{-iL\omega}\sum_{L'=0}^{+\infty}\bar{w}_{L'}e^{iL'\omega}d\omega.
    \end{align*}
    Provided that the sequence $(Lw_L)_{L\geq 0}$ is summable, $\frac{dW(\omega)}{d\omega}=\sum_{L=0}^{+\infty}-iLw_Le^{-i\omega L}$, which proves the result.
\end{proof}

\subsection{Fourier series}\label{appendix subsection Fourier series}
We recall basics of Fourier Series. For more about Fourier series and their applications, we refer the reader to \cite{serov2017fourier}.

\begin{definition}[Fourier series]
    Let $f: \mathbb{R}\rightarrow \mathbb{R}$ be a piecewise continuous and $2\pi$-periodic function. The Fourier series of $f$ is the series of functions 
    \[
    S(f) = \sum_{n=-\infty}^{+\infty}
c_n(f)e^{int},  
\]
where $c_n(f)$ are the Fourier coefficients of $f$, such that 
\[
c_n(f) = \frac{1}{2\pi}\int_{-\pi}^\pi f(t)e^{-int}dt.
\]
The partial sums of these series write
\[
S_n(f)(t) = \sum_{k=-n}^nc_k(f)e^{ikt}
\]
\end{definition}

\begin{theorem}[Dirichlet]
    Let $f$ be piecewise $\mathcal{C}^1$ and $2\pi$-periodic. Therefore, for every $x\in\mathbb{R}$, $S_n(f)(x)$ converges to 
    \[
    \frac{f(x+0) + f(x-0)}{2},
    \]
    where $f(x+0)$ (resp. $f(x-0)$) denotes the right-hand (resp. left-hand) limit of $f$ at $x$.
\end{theorem}

\paragraph{Remark:}If the function \( f \) is not \( 2\pi \)-periodic, its graph on the interval \([0, 2\pi]\) can be extended periodically over \(\mathbb{R}\). In this case, Dirichlet's theorem is applicable at potential discontinuities at \( 0 \) and \( 2\pi \).

\subsection{A natural pair for autocorrelation}\label{appendix subsection natural pair}

A natural parametrization is to represent autocorrelation with $\gamma(k) = \rho^{\vert k\vert}$ with $\vert\rho\vert < 1$, as done in the main paper. This models exponentially decreasing autocorrelation between data. The natural associated time-frequency pair to represent is 
\[
(\gamma(k), \Gamma(e^{i\omega})) = (\rho^{\vert k\vert}, \frac{1-\rho^2}{\vert 1-\rho e^{-i\omega}\vert^2})
.\]
Indeed, as $\vert\rho\vert<1$, the sequence \((\rho^{\vert k\vert}e^{ik\omega})_{k\in\mathbb{Z}}\) is summable, \(\gamma\) admits a Fourier transform that we denote~$\Gamma$. For $\omega\in\mathbb{R}$.
\begin{align*}
    \Gamma(e^{i\omega}) &= \sum_{k=-\infty}^{+\infty}\rho^{\vert k\vert}e^{-i\omega k} =\sum_{k=1}^{+\infty}\rho^k e^{i\omega k} + \sum_{k=0}^{+\infty}\rho^ke^{-i\omega k}\\
    &= \frac{1}{1-\rho e^{i\omega k}} -1 + \frac{1}{1-\rho e^{-i\omega k}} =\frac{1-\rho^2}{\vert 1-\rho e^{-i\omega}\vert^2}.
\end{align*}

\begin{figure}[h]
    \centering
    \includegraphics[width=1\linewidth]{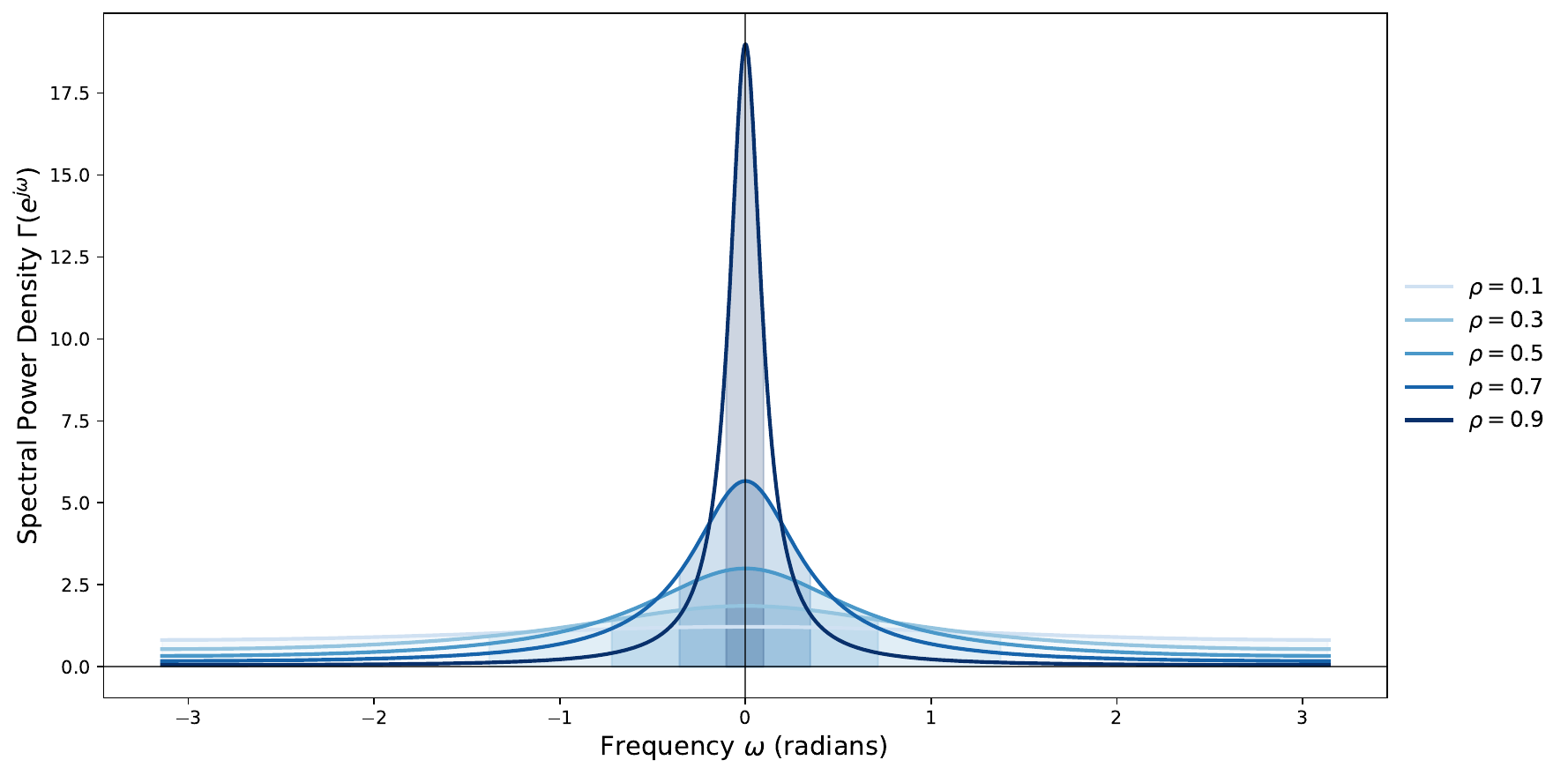}
    \caption{\textit{The autocorrelation factor $\rho$ determines the width of the spectral power density $\Gamma(e^{i\omega})$. The larger $\rho$, the narrower the spectral power density. This means that increasing $\rho$ in $\mathcal{L}_\text{freq}(c, d)$ narrows the bandwidth over which we evaluate the difference $\vert C(e^{i\omega}) - D(e^{i\omega})\vert^2$, leading to improved performance.}}
    \label{figure spectral power density}
\end{figure}

\section{Lower bound}

In this section, we provide the proofs of the two lower bounds.

\subsection{White noise case (Theorem~\ref{lower bound white noise})}\label{appendix subsection white noise loss}
We start by the representation of our loss function as a quadratic form.
\begin{proposition} \label{prop white noise loss}
    In the white noise case, the correlation factor $\rho$ is null. The loss $\mathcal{L}_\text{time}(c, d)$ writes 
    \[
    \mathcal{L}_\text{time}(c, d)= 1 + \sum_{k=0}^{+\infty}\vert c_k\vert^2 - 2\textnormal{Re}\big(\sum_{k=0}^{+\infty}c_kd_k\big),
    \]
    where $c_k=\sum_{s=1}^Sa_s^kb_s$. Therefore, the loss writes 
    \[
    \mathcal{L}_\text{time}(c, d) = 1 + \sum_{s, s'}^S\frac{b_s\bar{b}_{s'}}{1-a_s\bar{a}_{s'}} - 2\textnormal{Re}\big(\sum_{s=1}^Sb_sa_s^K\big).
    \]
\end{proposition}

\begin{proof}
    On the one hand, 
    \begin{align*}
        \sum_{k=0}^{+\infty}\vert c_k\vert^2 &= \sum_{k=0}^{+\infty}\big\vert\sum_{s=1}^Sa_s^kb_s\big\vert 
        =\sum_{k=0}^{+\infty}\sum_{s=1}^S\sum_{s'=1}^Sa_s^k\bar{a}_{s'}^kb_sb_{s'} \\
        &= \sum_{s=1}^S\sum_{s'=1}^Sb_sb_{s'}\sum_{k=0}^{+\infty}a_s^k\bar{a}_{s'}^k =\sum_{s=1}^S\sum_{s'=1}^Sb_sb_{s'}\frac{1}{1-a_s\bar{a}_{s'}}.
    \end{align*}
    On the other hand,
    \begin{align*}
        \textnormal{Re}\big(\sum_{k=0}^{+\infty}c_kd_k\big) &= c_Kd_K
        =\sum_{s=1}^Sb_sa_s^K.
    \end{align*}
    Hence the result.
\end{proof}

\begin{proposition}[Performance criterion]
    Minimizing the loss in Proposition \ref{prop white noise loss} boils down to maximizing the following performance criterion
    \[
    F_K = \sum_{s,s'=1}^S\bar{a}_s^K(C^{-1})_{ss'}a_{s'}^K,
    \]
    where $C_{ss'} = \frac{1}{1-a_s\bar{a}_{s'}}$.
\end{proposition}
\begin{proof}
    The loss $\mathcal{L}_\text{time}$ writes 
    \[
    1 + \langle\bar{b}, C\bar{b}\rangle - \langle\bar{b}, a^K\rangle -\langle a^K, \bar{b}\rangle.
    \]
    We thus want to maximize with respect to $a_s$ and $b_s$ the quantity
    \[
    \langle\bar{b}, a^K\rangle +\langle a^K, \bar{b}\rangle - \langle\bar{b}, C\bar{b}\rangle.
    \]
    This is convex and quadratic with respect to $b$, and the minimizer $\bar{b}^*$ is $C^{-1}a^K$, leading to the performance criterion
    \[
    F_K = \langle a^K, C^{-1}a^K\rangle = \sum_{s, s'=1}^S\bar{a}_s^K(C^{-1})_{ss'}a_{s'}^K.
    \]
\end{proof}

We can now move to the proof of Theorem \ref{lower bound white noise}, by first analyzing properties of the matrix $C$.

\paragraph{Linear algebra preview.} We use the similarities with Cauchy matrices and their so-called displacement structure~\citep{yang2003generalized,calvetti1996solution}.

Starting from
$$
C - \Diag( {a}) C \Diag(\bar{a})  = 1_S 1_S^\top,
$$
we get by  post multiplying by $\Diag(\bar{a})^{-1}$,
$$
C\Diag(\bar{a})^{-1}  - \Diag( {a}) C  = 1_S 1_S^\top\Diag(\bar{a})^{-1}
$$
and thus, by pre and post multiplying by $C^{-1}$:
$$
\Diag(\bar{a})^{-1}C^{-1}  - C^{-1}\Diag( {a})    = C^{-1}1_S 1_S^\top\Diag(\bar{a})^{-1}C^{-1},
$$
leading to
$$
\Diag(\bar{a})^{-1}C^{-1}\Diag( {a}) ^{-1}   - C^{-1} = C^{-1}1_S 1_S^\top\Diag(\bar{a})^{-1}C^{-1}\Diag( {a})^{-1}
= u v^*,
$$
with $ u = C^{-1} 1_S$ and $v = \Diag( \bar{a})^{-1}  C^{-1} \Diag(a)^{-1} 1_S$. This leads to a closed form expression for the inverse:
$$
(C^{-1})_{ss'} \big( \frac{1}{ \bar{a}_s a_{s'}}-1 \big) = u_s \bar{v}_{s'}. 
$$
We get
$$
v-u = \big[ \Diag( \bar{a})^{-1}  C^{-1} \Diag(a)^{-1} - C^{-1} \big] 1_S
= u v^\ast 1_S = u 1_S^\top\Diag(\bar{a})^{-1}C^{-1}\Diag( {a})^{-1}1_S,
$$
which leads to $ \ds v  = u ( 1 +1_S^\top\Diag(\bar{a})^{-1}C^{-1}\Diag( {a})^{-1}1_S )$. Moreover we can write
\BEAS
1_S^\top\Diag(\bar{a})^{-1}C^{-1}\Diag( {a})^{-1}1_S 
& = & 1_S^\top ( C^{-1} + u v^\ast) 1_S=  1_S^\top ( C^{-1} + C^{-1} 1_S v^\ast) 1_S\\
& = & 1_S^\top C^{-1} 1_S \cdot (1 + v^\ast 1_S )
\\
& = &  1_S^\top C^{-1} 1_S \cdot (1 + 1_S^\top\Diag(\bar{a})^{-1}C^{-1}\Diag( {a})^{-1}1_S ) , \EEAS
which leads to
$1_S^\top\Diag(\bar{a})^{-1}C^{-1}\Diag( {a})^{-1}1_S   = \frac{ 1_S^\top C^{-1} 1_S }{1- 1_S^\top C^{-1} 1_S }$, and thus
 $$ 1 +1_S^\top\Diag(\bar{a})^{-1}C^{-1}\Diag( {a})^{-1}1_S  = \frac{1}{1-1_S^\top C^{-1} 1_S} =
\frac{1}{1 - u^\top 1_S}.$$
Moreover, we have for any $z \in \mathbb{C}$, if all $a_s$ are distinct:
$$
\sum_{s'=1}^S \frac{ u_{s'}}{1 - z \bar{a}_{s'}} = 1 - \prod_{s'=1}^S \bar{a}_{s'} \prod_{s'=1}^S \frac{ a_{s'} - z}{1 - z \bar{a}_{s'}}
$$
(the two rational functions have the same degrees, the same poles and are equal for $z=a_1,\dots,a_S$),
which leads to for $z=0$,
$$
\sum_{s'=1}^S  { u_{s'}}=1_S^\top C^{-1} 1_S = \sum_{s'=1}^S   u_{s'} = 1 - \prod_{s'=1}^S | {a}_{s'}|^2,
$$
and thus $\ds 1 - 1_S^\top C^{-1} 1 = \prod_{s'=1}^S | {a}_{s'}|^2$.

We have, if $|z|=1$,
$$
\Big|\prod_{s'=1}^S \frac{ a_{s'} - z}{1 - z \bar{a}_{s'}}\Big|
= 1,
$$
which will be used in the bound (such expressions are typically referred to as Blaschke products~\citep{baratchart2016minimax}, and are known to have unit magnitude).

\paragraph{Proof of the lower bound (by upper bounding $F_K$).}
We have, using our linear algebra preview,
$$
F_K = \langle a^K, C^{-1} a^K \rangle
= \sum_{s,s'=1}^S \bar{a}_s^K (C^{-1})_{ss'} a_{s'}^K
= \sum_{s,s'=1}^S (\bar{a}_s  a_{s'})^{K+1} \frac{u_s \bar{v}_{s'} }{1 -  \bar{a}_s a_{s'}}.
$$
We get, using our linear algebra results,
$$
F_K - F_{K+1} = 
\sum_{s,s'=1}^S (\bar{a}_s  a_{s'})^{K+1} (1 -  \bar{a}_s a_{s'})\frac{u_s \bar{v}_{s'} }{1 -  \bar{a}_s a_{s'}}
=  \frac{1}{\prod_{s'=1}^S | {a}_{s'}|^2} \Big| \sum_{s=1}^S \bar{a}_s^{K+1}  u_s \Big|^2.
$$
This leads to 
\BEAS
F_K & = &  \sum_{L=K}^{+\infty}
( F_L - F_{L+1}) = 
\sum_{L=K+1}^{+\infty} \Big| \sum_{s=1}^S \bar{a}_s^L  u_s \Big|^2   \frac{1}{\prod_{s'=1}^S | {a}_{s'}|^2} 
.\EEAS
We have:
\BEAS
\sum_{L=K+1}^{+\infty} \Big| \sum_{s=1}^S \bar{a}_s^L  u_s \Big|^2 
& \leqslant &  \frac{1}{K+1} 
\sum_{L=0}^{+\infty} L \Big| \sum_{s'=1}^S \bar{a}_{s'}^L  u_{s'} \Big|^2 \mbox{ since } 1_{L \geqslant K+1} \leqslant \frac{L}{K+1}.
 \EEAS
 We consider the sequence $\ds w_L =  \sum_{s=1}^S \bar{a}_s^L  u_s$, with Fourier series
 $$
 W(\omega) = \sum_{L=0}^{+\infty} w_L e^{-i \omega L} 
 =  \sum_{s=1}^S \frac {u_s}{1-\bar{a}_s e^{-i \omega }}   =
  1 - \prod_{s'=1}^S \bar{a}_{s'} \prod_{s'=1}^S \frac{ a_{s'} - e^{-i\omega}}{1 -  e^{-i\omega} \bar{a}_{s'}}.
 $$
 We then use Proposition \ref{proposition semi parseval} to write:
 $$
 \sum_{L = 0 }^{+\infty}
 L | w_L|^2 =   \frac{i}{2\pi} \int_0^{2\pi} W'(\omega) \overline{W(\omega)}d\omega
, $$
 leading to
 \BEAS
&&\sum_{L=K+1}^{+\infty} \Big| \sum_{s=1}^S \bar{a}_s^L  u_s \Big|^2\\
& \leqslant &  \frac{1}{K+1} 
  \frac{i}{2\pi   }  \int_0^{2\pi} 
\frac{d}{d\omega} \Big[    - \prod_{s=1}^S \bar{a}_{s} \prod_{s=1}^S \frac{ a_{s} - e^{i\omega}}{1 - e^{i\omega} \bar{a}_{s}}
 \Big]
\overline{ \Big(
 1 - \prod_{s'=1}^S \bar{a}_{s'} \prod_{s'=1}^S \frac{ a_{s'} - e^{i\omega}}{1 - e^{i\omega} \bar{a}_{s'}}
 \Big)}
      d\omega
\\
& = &  \frac{1}{K+1} 
  \frac{i}{2\pi   }  \int_0^{2\pi} 
\frac{d}{d\omega} \Big[      \prod_{s=1}^S \bar{a}_{s} \prod_{s=1}^S \frac{ a_{s} - e^{i\omega}}{1 - e^{i\omega} \bar{a}_{s}}
 \Big]
\overline{ \Big(
    \prod_{s'=1}^S \bar{a}_{s'} \prod_{s'=1}^S \frac{ a_{s'} - e^{i\omega}}{1 - e^{i\omega} \bar{a}_{s'}}
 \Big)}
      d\omega.
\\
      \EEAS
      We now have, by taking derivatives of the product:
     \BEAS
     \frac{d}{d\omega} \Big[       \prod_{s=1}^S \frac{ a_{s} - e^{i\omega}}{1 - e^{i\omega} \bar{a}_{s}}
 \Big] & = & 
  \prod_{s=1}^S \frac{ a_{s} - e^{i\omega}}{1 - e^{i\omega} \bar{a}_{s}}
  \sum_{s=1}^S \frac{1 - e^{i\omega} \bar{a}_{s}}{ a_{s} - e^{i\omega}} 
   \frac{d}{d\omega} \Big[      \frac{ a_{s} - e^{i\omega}}{1 - e^{i\omega} \bar{a}_{s}}\Big]
\\
 & = & 
  \prod_{s=1}^S \frac{ a_{s} - e^{i\omega}}{1 - e^{i\omega} \bar{a}_{s}}
  \sum_{s=1}^S \frac{1 - e^{i\omega} \bar{a}_{s}}{ a_{s} - e^{i\omega}} 
   \frac{d}{d\omega} \Big[     \frac{1}{\bar{a}_s} +   \frac{a_s - \frac{1}{\bar{a}_s}}{1 - e^{i\omega} \bar{a}_{s}}\Big]
\\
 & = & 
  \prod_{s=1}^S \frac{ a_{s} - e^{i\omega}}{1 - e^{i\omega} \bar{a}_{s}}
  \sum_{s=1}^S \frac{1 - e^{i\omega} \bar{a}_{s}}{ a_{s} - e^{i\omega}} 
  \Big[
(1-|a_s|^2) \frac{ -i e^{i \omega}}{ (1 - e^{i\omega} \bar{a}_{s})^2}\Big]
\\
 & = & 
  \prod_{s=1}^S \frac{ a_{s} - e^{i\omega}}{1 - e^{i\omega} \bar{a}_{s}}
  \sum_{s=1}^S (1-|a_s|^2) \frac{-i  }{ |e^{-i\omega}  - \bar{a}_{s}|^2}
. \EEAS 
      This leads to, using the unit magnitude of $\frac{ a_{s} - e^{i\omega}}{1 - e^{i\omega} \bar{a}_{s}}$,
\BEAS
      F_K
      & \leqslant &  \frac{1}{K+1} 
  \frac{1}{2\pi   } 
   \sum_{s=1}^S ( 1- |a_s|^2) 
  \int_0^{2\pi}  \frac{1}{|a_s - e^{i\omega}|^2}
      d\omega =  \frac{S}{K+1},
\EEAS
using an explicit integration $\ds \frac{1}{2\pi   } 
    \int_0^{2\pi}  \frac{1}{|a_s - e^{i\omega}|^2}
      d\omega = \frac{1}{1-|a_s|^2}$.
    
 The approximation error $\mathcal{L}_\text{time}(c, d)$ is  thus 
$
1 - F_K$, 
which leads to the desired result.

\subsection{Autocorrelated case (Theorem~\ref{theorem 
autocorrelated lower bound})}
\label{proof auto}
We follow the same proof technique as for Theorem~\ref{lower bound white noise}, and compute first an explicit expression of the loss, this time, by introducing a new $a_s$, equal to $\rho$, with the introduction of new weights $w_s = b_s a_s / ( a_s - \rho)$ for $s \in \{1,\dots,S\}$, the weight $w_{S+1}$ being determined by the linear constraint.
\begin{lemma}
    In the autocorrelated case ($\rho \neq 0$), $\mathcal{L}_\text{time}(c, d)$ as in Eq.~\eqref{correlated time domain loss} writes 
    \begin{equation}
    1 - 2(1-\rho^2)\textnormal{Re}\big(\sum_{s=1}^{S+1}\frac{w_sa_s^k}{1-a_s\rho}\big) + (1-\rho^2)\sum_{s, s'}^{S+1}\frac{w_s\bar{w}_{s'}}{1-a_s\bar{a}_{s'}},
    \label{appendix constrained autocorrelated loss}
    \end{equation}    
    where $a_{S+1}=\rho$ and the constraint $\sum_{s=1}^{S+1}w_sa_s^{-1}=0$ holds. 
    \end{lemma}
\begin{proof}
   We aim to minimize 
    \[
    \sum_{k, k'}(c_k-d_k)(c_{k'}-d_{k'})\gamma(k-k'),
    \]
    where $\gamma(k-k')=\rho^{\vert k-k'\vert}$.
    Denoting $C(e^{i\omega}), D(e^{i\omega})$ and $\Gamma(e^{i\omega})$ the Fourier transforms of $(c_n), (d_n)$ and $(\gamma_n)$ respectively, Parseval's theorem yields 
    \[
    \sum_{k, k'}(c_k-d_k)(c_{k'}-d_{k'})\gamma(k-k') = \frac{1}{2\pi}\int_{-\pi}^\pi\big\vert C(e^{i\omega}) - D(e^{i\omega})\big\vert^2\Gamma(e^{i\omega})d\omega.
    \]

    We have $D(e^{i\omega})=e^{-iK\omega}$ (Fourier transform of a shifted Dirac at timestep K), and 
    \begin{align*}
        C(e^{i\omega}) &= \sum_{k=0}^{+\infty}\sum_{s=1}^Sb_sa_s^ke^{-i\omega k} = \sum_{s=1}^S\frac{b_s}{1-a_se^{-i\omega}},\\
        \Gamma(e^{i\omega})&=\sum_{k=-\infty}^{+\infty}\gamma(k)e^{-i\omega k} = \frac{1}{1 - \rho e^{-i\omega}}\frac{1-\rho^2}{1 - \rho e^{i\omega}}.
    \end{align*}
    The criterion becomes (with an error of $1$ if $C=0$):
\BEAS
&&\frac{1}{2\pi} \int_0^{2\pi} | D(e^{i\omega}) - C(e^{i\omega})|^2 \Gamma(e^{i\omega}) d\omega\\
& = & 
\frac{ 1-\rho^2}{2\pi} \int_0^{2\pi} \Big| D(e^{i\omega})\frac{1}{1 - \rho e^{-i\omega}} - C(e^{i\omega}) \frac{1}{1 - \rho e^{-i\omega}}\Big|^2  d\omega \\
& = & 1 - 
\frac{1-\rho^2}{2\pi} 2 {\textnormal{ Re}} \Big(\int_0^{2\pi} 
\overline{D(e^{i\omega})\frac{1}{1 - \rho e^{-i\omega}}}
C(e^{i\omega}) \frac{1}{1 - \rho e^{-i\omega}}
\Big) d\omega   \\
& & \hspace*{2cm} + \frac{1-\rho^2}{2\pi}  \int_0^{2\pi} \Big|C(e^{i\omega}) \frac{1}{1 - \rho e^{-i\omega}}\Big|^2  d\omega.
\EEAS 
We have 
$$
\frac{1}{1- a_se^{-i\omega}}\frac{1}{1 - \rho e^{-i\omega}}
= \frac{1}{a_s-\rho} \Big( \frac{a_s}{1 - a_s e^{-i\omega}} - \frac{\rho}{1 - \rho e^{-i\omega}} \Big),
$$
and thus
\BEAS
C(e^{i\omega}) \frac{1}{1 - \rho e^{-i\omega}}  & = &  \sum_{s=1}^{S}  
\frac{b_s}{a_s-\rho} \Big( \frac{a_s}{1 - a_s e^{-i\omega}} - \frac{\rho}{1 - \rho e^{-i\omega}} \Big) \\
 & = & \sum_{s=1}^{S+1} \frac{w_s  }{1- a_se ^{-i\omega}},
\EEAS
with $w_s = b_s a_s / ( a_s - \rho)$, $a_{S+1} = \rho$, and the constraint $\ds \sum_{s=1}^{S+1} w_sa_s^{-1}  = 0$.
The criterion becomes
\BEAS
& & 1 - 
(1-\rho^2)\sum_{s=1}^{S+1} 2 {\textnormal{Re}} \Big( \frac{w_s a_s^K}{1 - a_s \rho}
 \Big)   +(1-\rho^2) \sum_{s,s'=1}^{S+1} \frac{\bar{w}_s w_{s'}  }{1- a_s \bar{a}_s'},
\EEAS 
after straightforward computations.
\end{proof}

\paragraph{Proof of Theorem \ref{theorem autocorrelated lower bound}.}

The minimum with respect to $w$ in Eq.~\eqref{appendix constrained autocorrelated loss} with the constraint is greater than the unconstrained minimizer, equal to
\BEAS
H_K & =  & 1 - (1-\rho^2)\sum_{s,s'=1}^{S+1} 
\frac{ \bar{a}_s^K}{1 - \bar{a}_s \rho}\frac{ a_{s'}^K}{1 - a_{s'} \rho} (C^{-1})_{ss'},
\EEAS
where we recall that $C_{ss'} = \frac{1}{1-a_s\bar{a}_{s'}}$.

Using linear algebra properties from above with $S+1$ zeros and poles, we get
\BEAS
H_K&= & 1 -  (1-\rho^2)\sum_{s,s'=1}^{S+1}
\frac{ 1}{1 - \bar{a}_s \rho}\frac{1}{1 - a_{s'} \rho}  \frac{(\bar{a}_s a_{s'})^{K+1}u_s \bar{v}_{s'}}{1 - \bar{a}_s a_{s'}} 
\\
& = & 1 - (1-\rho^2)\frac{1}{\prod_{s=1}^{S+1} |a_s|^2 }
\sum_{s,s'=1}^{S+1} 
\frac{ 1}{1 - \bar{a}_s \rho}\frac{1}{1 - a_{s'} \rho}  \frac{(\bar{a}_s a_{s'})^{K+1}u_s \bar{u}_{s'}}{1 - \bar{a}_s a_{s'}} ,
\EEAS
where we recall that $u = C^{-1}1_S$ and $v = \text{Diag}(\bar{a})^{-1}C^{-1}\text{Diag}(a)^{-1}1_S$.

We have
\BEAS
H_{K+1} - H_K 
& = & \frac{1-\rho^2}{\prod_{s=1}^{S+1} |a_s|^2 }
\sum_{s,s'=1}^{S+1} 
\frac{ 1}{1 - \bar{a}_s \rho}\frac{1}{1 - a_{s'} \rho}   (\bar{a}_s a_{s'})^{K+1}u_s \bar{u}_{s'} 
\\
& = & \frac{1-\rho^2}{\prod_{s=1}^{S+1} |a_s|^2 }
\Big| \sum_{s=1}^{S+1}
\frac{ 1}{1 - \bar{a}_s \rho}   \bar{a}_s  ^{K+1}u_s  
\Big|^2,
 \EEAS
 leading to
 \BEAS
 H_K & = & \sum_{L=K}^{+\infty} ( H_L - H_{L+1} ) + 1 \\
 & = & 1 - \frac{1-\rho^2}{\prod_{s=1}^{S+1} |a_s|^2 }
 \sum_{L = K}^{+\infty}
 \Big| \sum_{s=1}^{S+1} 
\frac{ 1}{1 - \bar{a}_s \rho}   \bar{a}_s  ^{L} \bar{a}_su_s  
\Big|^2 \\
& \geqslant & 
1 - \frac{1}{K} \frac{1-\rho^2}{\prod_{s=1}^{S+1} |a_s|^2 }
 \sum_{L = 0 }^{+\infty}
 L\Big| \sum_{s=1}^{S+1}
\frac{ 1}{1 - \bar{a}_s \rho}   \bar{a}_s  ^{L} \bar{a}_s u_s  
\Big|^2 ,
 \EEAS
 using $1_{L \geqslant K} \leqslant \frac{L}{K}$.
 
 The sequence $\ds w_L = \sum_{s=1}^{S+1} 
\frac{ 1}{1 - \bar{a}_s \rho}   \bar{a}_s  ^{L} \bar{a}_su_s  $, has Fourier series
\BEAS
W(\omega) & = &  \sum_{L=0}^{+\infty} w_L e^{-i\omega L}
= \sum_{L=0}^{+\infty}   e^{-i\omega L}\sum_{s=1}^{S+1} 
\frac{ 1}{1 - \bar{a}_s \rho}   \bar{a}_s  ^{L}\bar{a}_s u_s \\
& = & \sum_{s=1}^{S+1}
\frac{ 1}{1 - \bar{a}_s \rho}   \frac{\bar{a}_s u_s}{1 - \bar{a}_s e^{-i\omega}} 
= \sum_{s=1}^{S+1} u_s \Big( 
\frac{ 1}{1 - \bar{a}_s \rho}   - \frac{1}{1 - \bar{a}_s e^{-i\omega}} \Big)\frac{1}{  \rho - e^{-i\omega}} \\
& = & 
\frac{1}{  \rho - e^{-i\omega}} \Big( \prod_{s=1}^{S+1} \bar{a}_s\Big) \Big( \prod_{s=1}^{S+1} \frac{a_s - e^{-i\omega}}{1-e^{-i\omega} \bar{a}_s}
-\prod_{s=1}^{S+1} \frac{a_s - \rho }{1- \rho \bar{a}_s}
\Big)
\\
 & = & 
\frac{1}{  \rho - e^{-i\omega}} \Big(\prod_{s=1}^{S+1} \bar{a}_s \Big) \prod_{s=1}^{S+1} \frac{a_s - e^{-i\omega}}{1-e^{-i\omega} \bar{a}_s},
 \EEAS
 because of the link between $u,C$ and rational functions.
 
We have:
\BEAS
 1 - H_K
 & \leqslant & 
 \frac{1-\rho^2}{K} \frac{1}{\prod_{s=1}^{S+1} |a_s|^2 }
 \sum_{L = 0 }^{+\infty}
 L | w_L|^2 \\
 & = & \frac{1-\rho^2}{K} \frac{1}{\prod_{s=1}^{S+1} |a_s|^2 }
\frac{i}{2\pi} \int_0^{2\pi} W'(\omega) \overline{W(\omega)}d\omega
\ \mbox{ using properties of Fourier Series,} \\
 & = & \frac{1-\rho^2}{K}  
\frac{i}{2\pi} \int_0^{2\pi} \frac{d}{d\omega} \Big(
\frac{1}{  \rho - e^{-i\omega}}   \prod_{s=1}^{S+1} \frac{a_s - e^{-i\omega}}{1-e^{-i\omega} \bar{a}_s}
\Big)
\overline{\frac{1}{  \rho - e^{-i\omega}}   \prod_{s=1}^{S+1} \frac{a_s - e^{-i\omega}}{1-e^{-i\omega} \bar{a}_s}}d\omega
\\
 & = & \frac{1-\rho^2}{K}  
\frac{i}{2\pi} \int_0^{2\pi}  
\frac{-i e^{-i\omega}}{  (\rho - e^{-i\omega})^2}   \prod_{s=1}^{S+1} \frac{a_s - e^{-i\omega}}{1-e^{-i\omega} \bar{a}_s}
\overline{\frac{1}{  \rho - e^{-i\omega}}   \prod_{s=1}^{S+1} \frac{a_s - e^{-i\omega}}{1-e^{-i\omega} \bar{a}_s}}d\omega
\\
& & \hspace*{1cm} + 
\frac{1-\rho^2}{K}  
\frac{i}{2\pi} \int_0^{2\pi} 
\frac{1}{  \rho - e^{-i\omega}} \frac{d}{d\omega} \Big(   \prod_{s=1}^{S+1} \frac{a_s - e^{-i\omega}}{1-e^{-i\omega} \bar{a}_s}
\Big)
\overline{\frac{1}{  \rho - e^{-i\omega}}   \prod_{s=1}^{S+1} \frac{a_s - e^{-i\omega}}{1-e^{-i\omega} \bar{a}_s}}d\omega.
\EEAS
Using the following identities,
\BEAS
 \frac{a_s - e^{-i\omega}}{1-e^{-i\omega} \bar{a}_s}
& = & \frac{1}{\bar{a}_s} + \frac{a_s - 1/ \bar{a}_s}{1-e^{-i\omega} \bar{a}_s}, \\
\frac{d}{d\omega} \Big( \frac{a_s - e^{-i\omega}}{1-e^{-i\omega} \bar{a}_s}
\Big) & = & \frac{a_s - 1/ \bar{a}_s}{(1-e^{-i\omega} \bar{a}_s)^2} \bar{a}_s (-i e^{-i\omega})
=  i e^{-i\omega}\frac{1-|a_s|^2 }{(1-e^{-i\omega} \bar{a}_s)^2}, \\
\Big| \frac{a_s - e^{-i\omega}}{1-e^{-i\omega} \bar{a}_s}
\Big|  & = & 1, \EEAS
we get
\BEAS
 1 - H_K & \leqslant & \frac{1-\rho^2}{K}  
\frac{1}{2\pi} \int_0^{2\pi}  
\frac{e^{-i\omega}}{  (\rho - e^{-i\omega})^2 (\rho - e^{i\omega})}     d\omega
\\
& & \hspace*{4cm} +
\frac{1-\rho^2}{K}  
\sum_{s=1}^{S+1}  ( 1- |a_s|^2) 
\frac{1}{2\pi} \int_0^{2\pi} 
\frac{1}{  |\rho - e^{-i\omega}|^2}  
 \frac{1}{|a_s - e^{-i\omega}|^2} d\omega
\\
& = & \frac{1-\rho^2}{K}  
\frac{1}{2\pi} \int_0^{2\pi}  
\frac{e^{-i\omega}}{  (\rho - e^{-i\omega})^2 (\rho - e^{i\omega})}     d\omega
\\
& & \hspace*{-1cm}
+
\frac{(1-\rho^2)^2}{K}  
\frac{1}{2\pi} \int_0^{2\pi} 
\frac{1}{  |\rho - e^{-i\omega}|^4}  d\omega
 +
\frac{1-\rho^2}{K}  
\sum_{s=1}^{S}  ( 1- |a_s|^2) 
\frac{1}{2\pi} \int_0^{2\pi} 
\frac{1}{  |\rho - e^{-i\omega}|^2}  
 \frac{1}{|a_s - e^{-i\omega}|^2} d\omega
\\
& = & - \frac{1}{K} \frac{1}{1-\rho^2}   +
 \frac{1}{K} \frac{1+\rho^2}{1-\rho^2}   + 
\frac{1-\rho^2}{K}  
\sum_{s=1}^{S}  ( 1- |a_s|^2) 
\frac{1}{2\pi} \int_0^{2\pi} 
\frac{1}{  |\rho - e^{-i\omega}|^2}  
 \frac{1}{|a_s - e^{-i\omega}|^2} d\omega  
 \EEAS
 by exact integration. Then, using $\ds\frac{1}{  |\rho - e^{-i\omega}|^2}   \leqslant \frac{1}{(1-\rho)^2}$, 
 and $\ds\frac{1}{2\pi} \int_0^{2\pi} 
  \frac{1}{|a_s - e^{-i\omega}|^2} d\omega  = \frac{1}{1-|a_s|^2}$, we get
  \BEAS
1-H_K
& \leqslant &  
 \frac{1}{K} \frac{\rho^2}{1-\rho^2}   + 
\frac{1+\rho}{1-\rho}\frac{S}{K}  \leqslant  \frac{1}{K} \frac{\rho}{1-\rho}   + 
\frac{2}{1-\rho}\frac{S}{K}
= \frac{1}{K} \frac{1}{1-\rho} ( \rho + 2 S).
 \EEAS
Thus, we get an approximation error greater than
$\displaystyle
\Big( 1 - \frac{1}{K} \frac{3S}{1-\rho} \Big)_+.
$ (since it is always nonnegative).

\section{Upper bound}

Here, we prove the results of Section~\ref{section upper bound} of the paper. We begin by proving the expression of $\mathcal{L}_\text{freq}$, to then justify the parametrization of the optimal $b_s$ in Eq.~\eqref{Param_of_the_as}. We finally compute the asymptotic loss in  Eq.~\eqref{Upper bound as and bs} and Theorem~\ref{convergence to window}.

\subsection{Loss in frequency domain}\label{appendix subsection frequency loss}

Here, we prove the expression of the counterpart of $\mathcal{L}_\text{time}$ in the frequency domain, $\mathcal{L}_\text{freq}$. More explicitly, we give a proof of Eq.~\eqref{Frequential loss copy task}.

\begin{proof}
    Denote $(z_k)$ the discrete-time filter such that 
    \[
    z_k = \sum_{k'=0}^{+\infty}(c_{k'} - d_{k'})\gamma(k-k').
    \]
    Therefore, $(z_k)$ is a convolution between $(c_k-d_k)$ and $\gamma_k$, 
    \[
    \sum_{k, k'=0}^{+\infty}(c_k - d_k)(c_{k'} - d_{k'})\gamma(k-k) = \sum_{k=0}^{+\infty}(c_k - d_k)z_k.
    \]
    According to Parseval's theorem and denoting $C(e^{i\omega}), D(e^{i\omega})$ and $\Gamma(e^{i\omega})$ the respective Fourier transforms, we have: 
    \begin{align*}
        \sum_{k=0}^{+\infty}(c_k - d_k)z_k &= \frac{1}{2\pi}\int_0^{2\pi}Z(\omega)\overline{(C(e^{i\omega}) - D(e^{i\omega}))}d\omega\\
        &= \frac{1}{2\pi}\int_0^{2\pi}\Gamma(e^{i\omega})(C(e^{i\omega}) - D(e^{i\omega}))\overline{(C(e^{i\omega}) - D(e^{i\omega}))}d\omega\\
        &= \frac{1}{2\pi}\int_0^{2\pi}\big\vert C(e^{i\omega}) - D(e^{i\omega})\big\vert^2\Gamma(e^{i\omega})d\omega,
    \end{align*}
    by the convolution property of the DTFT. Finally, $(d_k)$ being the shifted impulse filter, its Fourier transform is $D(e^{i\omega})=e^{-iK\omega}$. The Fourier transform $C(e^{i\omega})$ of $(c_k)$ is given by 
        \[
    C(e^{i\omega}) = \sum_{s=1}^S b_s \sum_{k=-\infty}^\infty \left(a_s e^{-i\omega}\right)^k = \sum_{s=1}^S \frac{b_s}{1 - a_s e^{-i\omega}}.
    \]
\end{proof}

\subsection{Parametrization of the optimal \texorpdfstring{$\boldsymbol{b_s}$}{bs}}

\label{appendix subsection asymptotic bs}

For the sake of conciseness, we sometimes denote $a$ and $b$ for the vectors $(a_s)$ and $(b_s)$ respectively. Before proving Lemma~\ref{Lemma param of bs}, we show three general lemmas on Fourier series and remainder of series. We will use them later in the proof of Lemma~\ref{Lemma param of bs}.

\begin{lemma}\label{lemma eigenvector Toeplitz}
Let \(\alpha \in \mathbb{C}\) and \(S \in \mathbb{N}\). Consider the infinite Toeplitz matrix \(T\) defined by $$T(s, s')~=~\frac{1}{2\alpha - \mathrm{i}(s-s')\pi}.$$ Then, \(\frac{2}{e^{2\alpha} - e^{-2\alpha}}\) is an asymptotic eigenvalue of  \(T\), associated with the eigenvector \(z = ((-1)^s)_{s \in \mathbb{Z}}\).
\end{lemma}

\begin{proof}
We compute the action of \(T\) on \(z\):
\[
(Tz)_s= \sum_{s'} T(s, s')(-1)^{s'} = \sum_{s'} \frac{e^{\mathrm{i}\pi s'}}{2\alpha - \mathrm{i}\pi s + \mathrm{i}\pi s'}.
\]

This expression resembles a Fourier series evaluated at \(\omega = \pi\). For \(k \in \mathbb{Z}\), consider:
\[
\frac{1}{2\pi} \int_0^{2\pi} e^{-\frac{2\alpha}{\pi}(\omega-\pi)} e^{-\mathrm{i}k\omega} \, d\omega = \frac{1}{2} \cdot \frac{e^{2\alpha} - e^{-2\alpha}}{2\alpha + \mathrm{i}k\pi}.
\]

Thus:
\begin{equation}
\label{eq:FS}
\frac{1}{2\alpha + \mathrm{i}k\pi} = \frac{1}{2\pi} \cdot \frac{2}{e^{2\alpha} - e^{-2\alpha}} \int_0^{2\pi} e^{-\frac{2\alpha}{\pi}(\omega-\pi)} e^{-\mathrm{i}k\omega} \, d\omega.
\end{equation}

We recognize this as the Fourier coefficient of the function \(f_\alpha(\omega) = \frac{2}{e^{2\alpha} - e^{-2\alpha}} e^{-\frac{2\alpha}{\pi}(\omega-\pi)}\). Therefore, according to Dirichlet's theorem, for \(s \in \mathbb{Z}\):
\[
f_{2\alpha - \mathrm{i}\pi s}(\pi) = \sum_{s'} \frac{e^{\mathrm{i}\pi s'}}{2\alpha - \mathrm{i}\pi s + \mathrm{i}\pi s'} = \sum_{s'} \frac{(-1)^{s'}}{2\alpha - \mathrm{i}\pi(s-s')}.
\]

Simplifying, we find:
\( \displaystyle
\frac{2}{e^{2\alpha - \mathrm{i}\pi s} - e^{-2\alpha + \mathrm{i}\pi s}} = \sum_{s'} \frac{(-1)^{s'}}{2\alpha - \mathrm{i}\pi(s-s')}.
\)
Finally, we obtain
\( \displaystyle
(Tz)_s= \frac{2}{e^{2\alpha} - e^{-2\alpha}} z_s,
\)
and thus:
\( \displaystyle
Tz = \frac{2}{e^{2\alpha} - e^{-2\alpha}} z,
\)
proving that \(\frac{2}{e^{2\alpha} - e^{-2\alpha}}\) is an eigenvalue associated with \(z = ((-1)^s)_{s \in \mathbb{Z}}\).

\end{proof}

Now, we prove two general results on remainders of alternating series. 

\begin{lemma}\label{appendix lemma 12.1}
We have:
    \begin{equation*}
        R_N = \sum_{n=N}^{+\infty}\frac{(-1)^n}{n} = \frac{(-1)^N}{2N} + \frac{1}{2}\sum_{n=N}^{+\infty}\frac{(-1)^n}{n(n+1)}.
    \end{equation*}
\end{lemma}

\begin{proof}
    On the one side,  by grouping two consecutive terms,
    \begin{align*}
    R_N + R_{N+1} &= \sum_{n=N}^{+\infty}\frac{(-1)^n}{n} + \sum_{n=N+1}^{+\infty}\frac{(-1)^n}{n}  = \sum_{n=N}^{+\infty} \Big\{ \frac{(-1)^n}{n} - \frac{(-1)^n}{n+1} \Big\}= \sum_{n=N}^{+\infty}\frac{(-1)^n}{n(n+1)},
    \end{align*}
and on the other side, 
\[
R_{N+1} = R_N - \frac{(-1)^N}{N}.
\]
Therefore, \( \displaystyle
2R_N = \frac{(-1)^N}{N} + \sum_{n=N}^{+\infty}\frac{(-1)^n}{n(n+1)}.
\)
\end{proof}

\begin{lemma}\label{appendix lemma 12.2}
For $\alpha\in\mathbb{C}$ such that Re$(\alpha)>0$,
we have
\begin{equation*}
    \sum_{n=N}^{+\infty} \frac{(-1)^n}{\alpha + i\pi n} = \frac{i}{\pi}\times\frac{(-1)^N}{2N} + o\big(\frac{1}{N}\big).
\end{equation*}
\end{lemma}

\begin{proof}
    We denote $R_N = \sum_{n=N}^{+\infty}\frac{(-1)^n}{\alpha - i\pi n}$ and $S_N = \sum_{n=N}^{+\infty}\frac{(-1)^n}{-i\pi n}$. Consequently, 
    \begin{align*}
        R_N - S_N &= \sum_{n=N}^{+\infty}(-1)^n\big(\frac{1}{\alpha - i\pi n} + \frac{1}{i\pi n}\big) = \alpha\sum_{n=N}^{+\infty}\frac{(-1)^n}{\alpha i\pi n + \pi^2 n^2},
        \end{align*}
        leading to
        \begin{align*} R_N &= S_N + \alpha\sum_{n=N}^{+\infty}\frac{(-1)^n}{\alpha i\pi n + \pi^2n^2}= -\frac{i}{\pi}\sum_{n=N}^{+\infty}\frac{(-1)^n}{n} + \alpha\sum_{n=N}^{+\infty}\frac{(-1)^n}{\alpha i\pi n + \pi^2n^2}\\
        &= -\frac{i}{\pi}\times\frac{(-1)^N}{2N} - \frac{i}{\pi}\sum_{n=N}^{+\infty}\frac{(-1)^n}{n(n+1)} + \frac{\alpha}{\pi}\sum_{n=N}^{+\infty}\frac{(-1)^n}{\alpha in + \pi n^2}, 
    \end{align*}
    where the last line stems from Lemma~\ref{appendix lemma 12.1}. The result then follows by classical results on the remainder of alternating series.\footnote{Alternating series' criterion: Let $(a_n)$ be a positive and decreasing sequence such that $a_n\rightarrow 0$. Then, the series $\sum_n(-1)^na_n$ converges. Denoting $R_n = \sum_{k=n+1}^{+\infty}(-1)^ka_k$, we have $\vert R_n\vert\leq a_{n+1}$.}
\end{proof}

\begin{proof}\textbf{(Lemma~\ref{Lemma param of bs})} \hspace*{.15cm}
We adopt the parametrization in Eq.~\eqref{Param_of_the_as} for the poles \((a_s)\) and aim to determine the optimal parameters \((b_s)\) under this configuration. Since $\mathcal{L}_\text{time}(c, d)$ is convex with respect to \((b_s)\), setting the gradient to zero yields:
\[
b = C^{-1} \bar{a}^K,
\]
where \(C_{ss'} = \frac{1}{1 - a_s \bar{a}_{s'}}\) for $s, s'$ in $\llbracket -T, T\rrbracket$ where we recall that $S = 2T+1$. 
Let's denote the matrix $M$ such that $M(s, s') = \frac{K}{2\alpha - i\pi(s-s')} + \frac{1}{2}$ for $s, s'$ in $\llbracket -T, T\rrbracket$. 

First, we remark that, for $s\in\llbracket-T, T\rrbracket$, 
\[
(a^K)_s= (-1)^se^{-\alpha}\in\rb.
\]
We derive the asymptotic expansion for the optimal $b$, using a coordinate-wise approach. We recall that we place ourselves in the case $1\ll S \ll K$. Let us denote $z=((-1)^s)_{s\in\llbracket -T, T\rrbracket}$. For $s\in\llbracket-T, T\rrbracket$,

\begin{align}
    \!\!(Cz)_s \!\! &= \!\!(Mz)_s + (Cz)_s - (Mz)_s\notag\\
    &=\!\!\sum_{s'=-T}^T\frac{(-1)^{s'}K}{2\alpha - i\pi(s-s')}\notag\\
    &\quad + \sum_{s'=-T}^T\frac{(-1)^{s'}}{2} + \sum_{s'=-T}^T(-1)^{s'}\big[\frac{1}{1-e^{-\frac{2\alpha}{K}}e^{\frac{i\pi(s-s')}{K}}}-\frac{K}{2\alpha - i\pi(s-s')}-\frac{1}{2}\big]\notag\\
    &=\!\!\sum_{s'=-T}^T\frac{(-1)^{s'}K}{2\alpha - i\pi(s-s')}  + \frac{1}{2} + \sum_{s'=-T}^T(-1)^{s'}\big[\frac{2\alpha-i\pi(s-s')-K(1-e^{-\frac{2\alpha}{K}}e^{\frac{i\pi(s-s')}{K}})}{(1-e^{-\frac{2\alpha}{K}}e^{i\pi(s-s')})(2\alpha - i\pi(s-s'))} - \frac{1}{2}\big]\notag\\
    &=\!\!\sum_{s'=-T}^T\frac{(-1)^{s'}K}{2\alpha - i\pi(s-s')}  \!+\! \frac{1}{2}\! \notag\\
    &\quad +\!\frac{1}{K}\!\sum_{s'=-T}^T \! (-1)^{s'}\big[\frac{\frac{4\alpha^3}{3} - \pi^2 \alpha (s-s')^2 + i \left(\pi^3\frac{(s-s')^3}{6} - 2\pi \alpha^2 (s-s') \right) + o(1)}{(2\alpha-i\pi(s-s')+o(1))(2\alpha-i\pi(s-s'))}\big].\label{Cz approx}
\end{align}
But, 
\[
\bigg[\frac{- \pi^2 \alpha (s-s')^2 + i \left(\pi^3\frac{(s-s')^3}{6} - 2\pi \alpha^2 (s-s') \right) + o(1)}{(2\alpha-i\pi(s-s')+o(1))(2\alpha-i\pi(s-s'))}\bigg] \underset{+\infty}{\sim}i\pi\frac{s-s'}{6},
\]
So there exists a sequence $(\epsilon_{s'})$ for $s'\in\mathbb{Z}$ such that $\epsilon_{s'}\rightarrow 0$ and 
\begin{align}
    &\sum_{s'=-T}^T \! (-1)^{s'}\big[\frac{\frac{4\alpha^3}{3} - \pi^2 \alpha (s-s')^2 + i \left(\pi^3\frac{(s-s')^3}{6} - 2\pi \alpha^2 (s-s') \right) + o(1)}{(2\alpha-i\pi(s-s')+o(1))(2\alpha-i\pi(s-s'))} \notag
    \\&=\sum_{s'=-T}^T(-1)^{s'}\frac{\alpha}{3} +\sum_{s'=-T}^T(-1)^{s'}\frac{i\pi(s-s')}{6}(1+\epsilon_{s'})\notag\\
    &\label{eq:equivalent partial sums 0(1)}=(-1)^T[\frac{\alpha}{3}-\frac{i\pi s}{6}] +\sum_{s'=-T}^T(-1)^{s'}\frac{i\pi(s-s')\epsilon_{s'}}{6} = O(1)
\end{align}
Note that we had to keep the constant term $\frac{4\alpha^3}{3}$ which corresponds to $s=s'$.

Plugging this into eq~\eqref{Cz approx}, this leads to
 \begin{align*}
    (Cz)_s &=\sum_{s'=-T}^T\frac{(-1)^{s'}K}{2\alpha - i\pi(s-s')}  + \frac{1}{2} + O(\frac{1}{K}) =(-1)^s\frac{2K}{e^{2\alpha}-e^{-2\alpha}} + O(\frac{K}{T}),
 \end{align*}   
where we used Lemma~\ref{appendix lemma 12.2}.
We can deduce from this coordinate-wise equation that 
\[
Cz = \frac{2K}{e^{2\alpha}-e^{-2\alpha}}z + O(\frac{K}{T}).
\]
We finally use the bounded nature of $C$'s condition number\footnote{This is due to the link between eigenvalues of Toeplitz matrices and the Fourier series of the first row~\cite{gray2006toeplitz}, and the relationship $C_{ss'} = \frac{1}{1-\exp(-2\alpha/K) \exp( i \pi (s-s')/K)}
\sim \frac{K}{2 \alpha - i\pi ( s-s')}$ together with Eq.~\eqref{eq:FS}.} to apply $C^{-1}$ and to show:
\[ \displaystyle
C^{-1}\bar{a}^K \sim \frac{e^{2\alpha} - e^{-2\alpha}}{2K}z.
\]
This is valid when $T\rightarrow+\infty, T/K\rightarrow 0$.
\end{proof}

\begin{figure}[ht]
    \centering
    \includegraphics[width=1\linewidth]{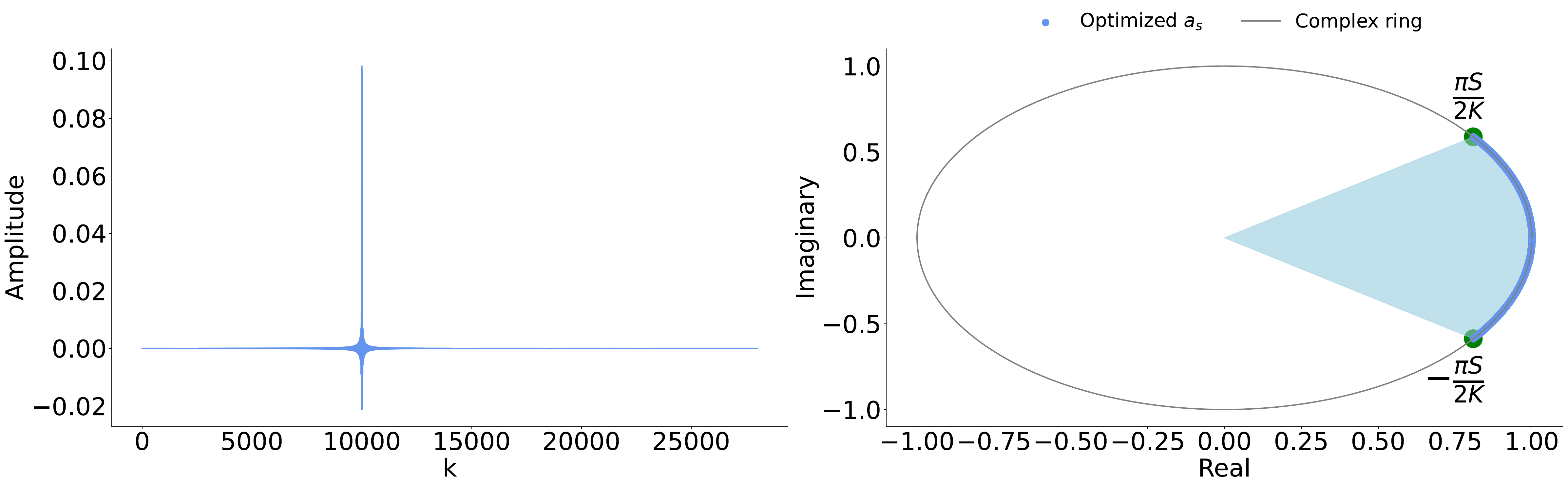}
    \caption{\textit{Left: Real values of filter in Eq.~\eqref{definition new filter} in the time-domain. Right: Positions of the $a_s$ on the unit disk. In this case, $S=100, K=10000$. The $x$-axis on the left images represents the timesteps. The representation in Eq.~\eqref{definition new filter} concentrates the poles $a_s$ on a slice of the unit disk, whose size depends on the ratio $\frac{S}{K}$. The $a_s$ operate in pairs of complex conjugates, ensuring that the final filter remains real. Each $a_s$ approximates a single oscillation of the complex exponential, with the oscillations spaced by a distance of $\frac{\pi}{K}$. Therefore when $K$ increases (for fixed $S$), the slice size decreases, imposing smaller phase shifts to capture long-range dependencies in the data (see \cite{orvieto2023resurrecting}). This parametrization allows to build filters than can look far back in time.}}
    \label{fig:time domain filter}
\end{figure}

\subsection{Upper bound of the loss}

In this section, we will prove Theorem~\ref{thm upper bound} through an asymptotic expansion. We will use general results on the remainder of alternate series in Lemmas~\ref{appendix lemma 12.1} and~\ref{appendix lemma 12.2}. We also start with a lemma very similar to Lemma~\ref{lemma eigenvector Toeplitz}. We use all of them later in the proof of Theorem~\ref{thm upper bound}.

\begin{lemma}\label{appendix lemma fourier series}
    Let $\alpha\in\mathbb{C}$ such that $\textnormal{Re}(\alpha)>0$. Then,
    \( \displaystyle     \sum_{s=-\infty}^{+\infty}\frac{(-1)^s}{2\alpha - i\pi s} = \frac{2}{e^{2\alpha}-e^{-2\alpha}}.
    \)
\end{lemma}

\begin{proof}
We consider $    \sum_{s=-\infty}^{+\infty}\frac{(-1)^s}{2\alpha - i\pi s}$.
    This looks like a Fourier series evaluated in $\omega=\pi$. We will look for a function $f_\alpha$ such that $c_s(f_\alpha) = \frac{1}{2\alpha -i\pi s}$. Denoting $f_\alpha$ such that $f_\alpha(\omega) = \frac{2e^{\frac{-2\alpha}{\pi}(\omega-\pi)}}{e^{2\alpha}-e^{-2\alpha}}$, we have,
\begin{align*}
    c_s(f_\alpha) &=\frac{1}{2\pi}\int_{0}^{2\pi}\frac{2e^{\frac{2\alpha}{\pi}(\omega-\pi)}}{e^{2\alpha}-e^{-2\alpha}}e^{-is\omega}d\omega=\frac{1}{\pi}\frac{1}{(e^{2\alpha}-e^{-2\alpha})(\frac{2\alpha}{\pi}-is)}(e^{2\alpha}-e^{-2\alpha})=\frac{1}{2\alpha - i\pi s}.
\end{align*}

Therefore, using Dirichlet's theorem, $f_\alpha$ is the appropriate function and
\begin{align*}
   &f_{\alpha}(\pi) = \sum_{s} \frac{(-1)^{s}}{2\alpha -i\pi s}  \Leftrightarrow\sum_{s=-\infty}^{+\infty} \frac{(-1)^{s}}{2\alpha -i\pi s} = \frac{2}{e^{2\alpha}-e^{-2\alpha}}. 
\end{align*}
\end{proof}

\begin{proof}\textbf{(Theorem~\ref{thm upper bound})}\label{proof thm 15} \hspace*{.2cm}
    We recall that we have the following asymptotic representations for our filter:

\begin{align*}
    a_s &= e^{-\frac{\alpha}{K}}e^{i\frac{s\pi}{K}}, s\in\llbracket-T, T\rrbracket\\
    b_s &= \frac{e^{-\alpha}(e^{2\alpha}-e^{-2\alpha})}{2K}(-1)^s,s\in\llbracket-T, T\rrbracket,
\end{align*}
and that $\mathcal{L}_\text{time}(c, d)$ can therefore write:
\begin{align*}
    \mathcal{L}_\text{time}(c, d) &= \sum_{s, s'=-T}^T\frac{b_s\bar{b}_{s'}}{1-a_s\bar{a}_{s'}} - 2\sum_{s=-T}^Tb_sa_s^K + 1\\
    &= \frac{e^{-2\alpha}(e^{2\alpha}-e^{-2\alpha})^2}{4K^2}\sum_{s, s'=-T}^T\frac{(-1)^{s+s'}}{1 - e^{-\frac{2\alpha}{K}}e^{(s-s')i\frac{\pi}{K}}} \\
    &\qquad - 2\sum_{s=-T}^T\frac{(-1)^se^{-2\alpha}(e^{2\alpha}-e^{-2\alpha})(-1)^s}{2K} + 1\\
    &=  \frac{e^{-2\alpha}(e^{2\alpha}-e^{-2\alpha})^2}{4K^2}\sum_{s, s'=-T}^T\frac{(-1)^{s+s'}}{1 - e^{-\frac{2\alpha}{K}}e^{(s-s')i\frac{\pi}{K}}} - \frac{e^{-2\alpha}(e^{2\alpha}-e^{-2\alpha})S}{K} + 1.
\end{align*}

First, we prove that 
\begin{equation}\label{approx O(1) loss}
    \sum_{s, s'=-T}^T\frac{(-1)^{s+s'}}{1 - e^{\frac{-2\alpha}{K}}e^{(s-s')i\frac{\pi}{K}}} = \sum_{s=-T}^T(-1)^sK\sum_{s'=-T}^T\frac{e^{i\pi s'}}{2\alpha - i\pi(s-s')} + O(1),
\end{equation}
when $T\rightarrow +\infty, T/K \rightarrow 0$. Let us compute:
\begin{align*}
        &\sum_{s, s'=-T}^T\frac{(-1)^{s+s'}}{1 - e^{\frac{-2\alpha}{K}}e^{(s-s')i\frac{\pi}{K}}} - \sum_{s=-T}^T(-1)^sK\sum_{s'=-T}^T\frac{e^{i\pi s'}}{2\alpha - i\pi(s-s')}\\
        =&\sum_{s, s'=-T}^T(-1)^{s+s'}\big[\frac{1}{1-e^{-\frac{2\alpha}{K}}e^{i\frac{\pi}{K}(s-s')}} - \frac{K}{2\alpha - i\pi(s-s')}\big]\\
        =&\sum_{s, s'=-T}^T(-1)^{s+s'}\big[\frac{2\alpha -i\pi(s-s')-K[1-e^{-\frac{2\alpha}{K}}e^{i\frac{\pi}{K}(s-s')}]}{(1-e^{-\frac{2\alpha}{K}}e^{i\frac{\pi}{K}(s-s')})(2\alpha - i\pi(s-s'))}\big].\\
\end{align*}
Let us finish the computation:
\begin{align*}
        =&\sum_{s, s'=-T}^T(-1)^{s+s'}\frac{2\alpha -i\pi(s-s')-[2\alpha - i\pi(s-s')-\frac{4\alpha^2}{2K}+\frac{\pi^2}{2K}(s-s')^2+\frac{4i\pi\alpha}{K}(s-s') +o(\frac{1}{K})]}{(1-e^{-\frac{2\alpha}{K}}e^{i\frac{\pi}{K}(s-s')})(2\alpha - i\pi(s-s'))}\\
        =&\sum_{s, s'=-T}^T\frac{(-1)^{s+s'}}{K}\times\frac{2\alpha^2 - \frac{\pi^2}{2}(s-s')^2 - 2i\pi\alpha(s-s')+o(1)}{(1-e^{-\frac{2\alpha}{K}}e^{i\frac{\pi}{K}(s-s')})(2\alpha - i\pi(s-s'))}\\
        =&\sum_{n=-2T}^{2T}\frac{(-1)^n}{K}\times\frac{2\alpha^2-\frac{\pi^2n^2}{2}-2i\pi\alpha n+o(1)}{(1-e^{-\frac{2\alpha}{K}}e^{i\frac{\pi}{K}n})(2\alpha - i\pi n)}(2T + 1-\vert n\vert) \text{ with the change of variable $n=s-s'$}\\
        =&\sum_{n=-2T}^{2T}\frac{(-1)^n}{K} (2T+1-\vert n\vert)f(n) \text{ with $f(n) = \frac{2\alpha^2-\frac{\pi^2n^2}{2}-2i\pi\alpha n+o(1)}{(1-e^{-\frac{2\alpha}{K}}e^{i\frac{\pi}{K}n})(2\alpha - i\pi n)}$}.
\end{align*}
But notice that \begin{align*}
    f(n) &= \frac{2\alpha^2-\frac{\pi^2n^2}{2}-2i\pi\alpha n+o(1)}{(1-e^{-\frac{2\alpha}{K}}e^{i\frac{\pi}{K}n})(2\alpha - i\pi n)} = \frac{K(2\alpha^2-\frac{\pi^2n^2}{2}-2i\pi\alpha n+o(1))}{(2\alpha - i\pi n +o(1))(2\alpha - i\pi n)},\\
\end{align*}
therefore
\begin{align*}
\sum_{n=-2T}^{2T}\frac{(-1)^n}{K} (2T+1-\vert n\vert)f(n) &= \sum_{n=-2T}^{2T}(-1)^n (2T+1-\vert n\vert) \frac{2\alpha^2-\frac{\pi^2n^2}{2}-2i\pi\alpha n+o(1)}{(2\alpha - i\pi n +o(1))(2\alpha - i\pi n)}\\
&\sim\sum_{n=-2T}^{2T}(-1)^n(2T+1-\vert n\vert)\frac{-\frac{\pi^2n^2}{2}}{-\pi^2n^2}\sim\frac{1}{2},
\end{align*}
using the same reasoning as for equation ~\eqref{eq:equivalent partial sums 0(1)}.
This shows 
\begin{equation}\label{O(1) intermediate}
    \sum_{n=-2T}^{2T}\frac{(-1)^n}{K}(2T+1-\vert n\vert)f(n) = O(1).
\end{equation}
We can therefore conclude:
\begin{align*}
    &\mathcal{L}_{\text{time}}(c,d) = 1 - \frac{e^{-2\alpha}(e^{2\alpha}-e^{-2\alpha})S}{K} + \frac{e^{-2\alpha}(e^{2\alpha}-e^{-2\alpha})^2}{4K^2}\times\sum_{s=-T}^T(-1)^sK\sum_{s'=-T}^T\frac{e^{i\pi s'}}{2\alpha - i\pi(s-s')} + O(\frac{1}{K^2})\\
    &=1 - \frac{e^{-2\alpha}(e^{2\alpha}-e^{-2\alpha})S}{K} + \frac{e^{-2\alpha}(e^{2\alpha}-e^{-2\alpha})^2}{4K}\big[\sum_{s=-T}^T(-1)^s\big(\sum_{s'=-\infty}^{+\infty}\frac{(-1)^{s'}}{(2\alpha-i\pi s)+i\pi s'} + O(\frac{1}{T})\big)\big] + O(\frac{1}{K^2})\\
    &=1 - \frac{e^{-2\alpha}(e^{2\alpha}-e^{-2\alpha})S}{K} + \frac{e^{-2\alpha}(e^{2\alpha}-e^{-2\alpha})^2}{4K}\big[\sum_{s=-T}^T(-1)^s\frac{2\times(-1)^s}{e^{2\alpha}-e^{-2\alpha}} +O(\frac{1}{T})\big] + O(\frac{1}{K^2})\\
    &= 1 - \frac{e^{-2\alpha}(e^{2\alpha}-e^{-2\alpha})S}{K} + \frac{e^{-2\alpha}(e^{2\alpha}-e^{-2\alpha})S}{2K} + O(\frac{1}{TK})\\
    &=  1 - \frac{e^{-2\alpha}(e^{2\alpha}-e^{-2\alpha})S}{2K} + O(\frac{1}{TK}),
\end{align*}
where we used Lemmas~\ref{appendix lemma 12.2} and~\ref{appendix lemma fourier series}.
\end{proof}

\subsection{Proof of Theorem~\ref{convergence to window}}

In this section, we give the proof for Theorem~\ref{convergence to window} that describes the asymptotic behavior of the transfer function $C(e^{i\omega})$. First, we refer the reader to Lemmas~\ref{appendix lemma 12.1} and~\ref{appendix lemma 12.2} where we derive results on the remainder of some series. Then in Lemmas~\ref{appendix lemma 12.3} and~\ref{appendix lemma 12.4}, we derive an asymptotic new form for the transfer function. We combine all these lemmas to prove Theorem~\ref{convergence to window}.

\begin{lemma}\label{appendix lemma 12.3}
For $\alpha$ real and positive and $\Omega$ real,
\begin{equation*}
    \frac{1}{2}\sum_{s=-T}^T\frac{e^{-\alpha}(e^{2\alpha}-e^{-2\alpha})(-1)^s}{K(1 - e^{-\alpha/K}e^{i\pi(\frac{s}{K}-\frac{\Omega}{K})})} = \frac{1}{2}\sum_{s=-T}^T\frac{e^{-\alpha}(e^{2\alpha}-e^{-2\alpha})(-1)^s}{\alpha - i\pi (s-\Omega)} + O\big(\frac{1}{K}\big),
\end{equation*}
as $T\rightarrow+\infty, T/K\rightarrow 0$.
    
\end{lemma}

\begin{proof}
    \begin{align*}
         &\sum_{s=-T}^T\frac{(-1)^s}{K(1 - e^{-\alpha/K}e^{i\pi(\frac{s}{K}-\frac{\Omega}{K})})} - \sum_{s=-T}^T\frac{(-1)^s}{\alpha - i\pi (s-\Omega)}\\
         &= \sum_{s=-T}^T(-1)^s\big[\frac{1}{K(1-e^{-\alpha/K}e^{i\frac{\pi}{K}(s-\Omega)})} - \frac{1}{\alpha -i\pi (s - \Omega)}\big]\\
         &= \sum_{s=-T}^T(-1)^s\frac{\alpha -i\pi(s-\Omega)-K(1 - e^{-\alpha/K}e^{\frac{i\pi}{K}(s-\Omega)})}{K\big(1-e^{-\alpha/K}e^{i\frac{\pi}{K}(s-\Omega)}\big)\big(\alpha - i\pi (s-\Omega)\big)}\\
         &= \sum_{s=-T}^T(-1)^s\frac{\alpha - i\pi(s-\Omega) - K\big(\frac{\alpha}{K} - \frac{i\pi}{K}(s-\Omega) - \frac{\alpha^2}{2K^2} + \frac{\pi^2}{2K^2}(s-\Omega)^2 + \frac{i\pi\alpha}{K^2}(s-\Omega)+o(\frac{1}{K^2})\big)}{K\big(1-e^{-\alpha/K}e^{i\frac{\pi}{K}(s-\Omega)}\big)\big(\alpha - i\pi (s-\Omega)\big)}\\
         &= \sum_{s=-T}^T(-1)^s\frac{\frac{\alpha^2}{2K} - \frac{\pi^2}{2K}(s-\Omega)^2 - \frac{i\pi\alpha}{K}(s-\Omega) + o(\frac{1}{K})}{K\big(1-e^{-\alpha/K}e^{i\frac{\pi}{K}(s-\Omega)}\big)\big(\alpha - i\pi (s-\Omega)\big)}\\
         &=\frac{1}{K}\sum_{s=-T}^T(-1)^s\frac{\frac{\alpha^2}{2} - \frac{\pi^2}{2}(s-\Omega)^2-i\pi\alpha(s-\Omega)+o(1)}{[\alpha-i\pi(s-\Omega)+o(1)][\alpha-i\pi(s-\Omega)]}=O\big(\frac{1}{K}\big).
    \end{align*}
    We used a similar argument as the one in eq.~\eqref{eq:equivalent partial sums 0(1)} and eq.~\eqref{O(1) intermediate}.
    The constant $\frac{e^{-\alpha}(e^{2\alpha}-e^{-2\alpha})}{2}$ does not impact our computation. 
\end{proof}
    
\begin{lemma}\label{appendix lemma 12.4}
For $\alpha$ real and positive, 
\[
\frac{1}{2}\sum_{s=-T}^T\frac{(-1)^se^{-\alpha}(e^{2\alpha}-e^{-2\alpha})}{\alpha - i\pi(s-\Omega)} \underset{S/K\rightarrow 0}{\underset{S\rightarrow+\infty}{\sim} }
\begin{cases}
    \frac{e^{-\alpha}(e^{2\alpha}-e^{-2\alpha})}{2}\times\frac{i(-1)^{T+1}\times 2\lfloor\Omega\rfloor}{2\pi(\lfloor\Omega\rfloor-T)(\lfloor\Omega\rfloor+T)} & \text{if } \vert\Omega\vert > T, \\
\frac{e^{-\alpha}(e^{2\alpha}-e^{-2\alpha})}{e^{\alpha}e^{i\pi\Omega}-e^{-\alpha}e^{-i\pi\Omega}} & \text{if } \vert\Omega\vert < T.
\end{cases}
\] 
\end{lemma}

\begin{proof}
    We decompose $\Omega = \lfloor\Omega\rfloor + \beta = n + \beta$, and look at two different cases.
    
    \paragraph{First case:}$\vert\Omega\vert < T$. We have:
    
\begin{align*}
        &\frac{1}{2}\sum_{s=-T}^T(-1)^s\frac{e^{-\alpha}(e^{2\alpha}-e^{-2\alpha})}{\alpha - i\pi(s-n-\beta)}\\
        &=\frac{(-1)^n}{2}\sum_{s=-(T+n)}^{T-n}\frac{e^{-\alpha}(e^{2\alpha}-e^{-2\alpha})}{\alpha - i\pi(s-\beta)} =\frac{(-1)^n}{2}\sum_{s=-(T+n)}^{T-n}\frac{e^{-\alpha}(e^{2\alpha}-e^{-2\alpha})}{\tilde{\alpha} - i\pi s} \text{ where }\tilde{\alpha} = \alpha+i\pi\beta\\
        &=\frac{(-1)^n}{2}\sum_{s=-\infty}^{+\infty}\frac{e^{-\alpha}(e^{2\alpha}-e^{-2\alpha})}{\tilde{\alpha}-i\pi s}\\
        &+ \big[\frac{(-1)^n}{2}\sum_{s=-(T+n)}^{T-n}\frac{e^{-\alpha}(e^{2\alpha}-e^{-2\alpha})}{\tilde{\alpha}-i\pi s} -  \frac{(-1)^n}{2}\sum_{s=-\infty}^{+\infty}\frac{e^{-\alpha}(e^{2\alpha}-e^{-2\alpha})}{\tilde{\alpha}-i\pi s}\big]\\
        &=\frac{(-1)^n}{2}\sum_{s=-\infty}^{+\infty}\frac{e^{-\alpha}(e^{2\alpha}-e^{-2\alpha})}{\tilde{\alpha}-i\pi s} -\frac{i}{\pi}\times\frac{(-1)^{T-n+1}}{2(T-n+1)} + \frac{i}{\pi}\times\frac{(-1)^{T+n+1}}{2(T+n+1)} + o\big(\frac{1}{T-n}\big) \text{ (Lemma~\ref{appendix lemma 12.2})}\\
        &= (-1)^n\frac{e^{-\alpha}(e^{2\alpha}-e^{-2\alpha})}{e^{\alpha}e^{i\pi\beta} - e^{-\alpha}e^{-i\pi\beta}} + O\big(\frac{1}{{T}}\big) = \frac{e^{-\alpha}(e^{2\alpha}-e^{-2\alpha})}{e^{\alpha}e^{i\pi\Omega} - e^{-\alpha}e^{-i\pi\Omega}} +  O\big(\frac{1}{{T}}\big).
\end{align*}

    \paragraph{Second case:} $\vert\Omega\vert > T$. We consider again: 
    \begin{align*}
         \frac{(-1)^n}{2}\sum_{s=-(T+n)}^{T-n}\frac{e^{-\alpha}(e^{2\alpha}-e^{-2\alpha})}{\tilde{\alpha} - i\pi s} 
        &=\frac{(-1)^ne^{-\alpha}(e^{2\alpha}-e^{-2\alpha})}{2}\big[\sum_{s=-\infty}^{T-n}\frac{1}{\tilde{\alpha} - i\pi s} - \sum_{s=-\infty}^{-T-n-1}\frac{1}{\tilde{\alpha} - i\pi s}\big]\\
        &=\frac{(-1)^ne^{-\alpha}(e^{2\alpha}-e^{-2\alpha})}{2}\big[\sum_{s=n-T}^{+\infty}\frac{1}{\tilde{\alpha}+i\pi s} - \sum_{s=T+n+1}^{+\infty}\frac{1}{\tilde{\alpha}+i\pi s}
        \big]\\
        &= \frac{e^{-\alpha}(e^{2\alpha}-e^{-2\alpha})}{2}\times \frac{i(-1)^{T+1}\times 2n}{2\pi(n-T)(T+n)} + o\big(\frac{1}{T}\big) \text{\quad (Lemma~\ref{appendix lemma 12.2}) }.
    \end{align*}
\end{proof} 

\begin{proof}\textbf{(Theorem~\ref{convergence to window})} \hspace*{.2cm}
For $\Omega\in\mathbb{R}$, 
\begin{align*}
    C(e^{i\omega}) &= \sum_{s=-T}^T\frac{e^{-\alpha}(e^{2\alpha}-e^{-2\alpha})}{2K}\frac{(-1)^s}{1-e^{-\alpha/K}e^{i\frac{\pi}{K}(s-\Omega)}}\\
    &= \sum_{s=-T}^T\frac{e^{-\alpha}(e^{2\alpha}-e^{-2\alpha})}{2}\times\frac{(-1)^s}{\alpha - i\pi(s-\Omega)}\\
    &+\big[\sum_{s=-T}^T\frac{e^{-\alpha}(e^{2\alpha}-e^{-2\alpha})}{2K}\times\frac{(-1)^s}{1-e^{-\alpha/K}e^{i\frac{\pi}{K}(s-\Omega)}} - \sum_{s=-T}^T\frac{e^{-\alpha}(e^{2\alpha}-e^{-2\alpha})}{2}\frac{(-1)^s}{\alpha - i\pi(s-\Omega)}
    \big].
\end{align*}
We then use Lemma~\ref{appendix lemma 12.3} and Lemma~\ref{appendix lemma 12.4} to conclude.
    
\end{proof}

\section{Experiments}

In this section, we present a series of experiments designed to validate our theoretical findings in a practical setting. Specifically, we assess whether our conclusions hold when transitioning from an idealized infinite-data framework to real-world scenarios with a limited number of samples. 

Let us first introduce the linear recurrent neural network (RNN) used in our study. It is defined by the following recurrence relations:
\begin{align*}
    h_0 &= 0, \\
    x_{t+1} &= Ax_t + Bu_{t+1}, \\
    y_t &= Cx_t,
\end{align*}
where $x_t \in \mathbb{R}^{d_{\text{hidden}}}$ represents the hidden state, $u_t \in \mathbb{R}$ is the input, and $y_t \in \mathbb{R}$ is the output. The network parameters consist of $A \in \mathbb{C}^{d_{\text{hidden}} \times d_{\text{hidden}}}$, $B \in \mathbb{C}^{d_{\text{hidden}} \times 1}$, and $C \in \mathbb{C}^{d_{\text{hidden}}}$. 

Without loss of generality, we adopt a diagonal representation for the matrix $A$. The choice of its initial eigenvalues depends on the specific experiment: we either use a random initialization or employ the structured initialization given by Eq.~\eqref{Param_of_the_as}. 

In the simple experiments conducted below, the objective is to learn a single filter. Consequently, there is no need to decompose the matrix $A$ into multiple diagonal blocks. The matrix $C$ is initialized as:
\[
C_{\text{init}} = \begin{pmatrix} 1 & \dots & 1 \end{pmatrix} \in \mathbb{R}^{d_{\text{hidden}} \times 1}.
\]
and the entries of $B$ are initialized given by Eq.~\eqref{Param_of_the_bs}.\newline

The synthetic dataset consists of autoregressive sequences $X = (u_1, u_2, \dots, u_N)$ of length $N$, generated as:
\begin{equation}
    u_n = \rho u_{n-1} + \epsilon_k, \quad \epsilon_k \sim \mathcal{N}(0, 1 - \rho^2), \quad u_1 \sim U(0,1).
\end{equation}
The objective is to learn a mapping with linear recurrences $f: X \to Y$, where the target is given by:
\begin{equation}
    Y = u_{t^*}
\end{equation}
This corresponds to learning a shift of $N - t^*$ with finite samples.

\subsection{Random initialization vs. Shift-K initialization}

In this first set of experiments, we analyze the impact of initializing the complex diagonal entries $a_s$ of the linear RNN using phases that are uniformly distributed over a segment of the unit disk, with a constant radius close to 1, as described in the parametrization in Eq.~\eqref{Param_of_the_as}. Additionally, the parameters $b_s$ are initialized following the parametrization given in Eq.~\eqref{Param_of_the_bs}. We call this initialization the shift-$K$ initialization. We compare this approach to a standard random initialization to evaluate potential benefits in terms of performance and stability.

\begin{table}[ht]
    \centering
    \begin{tabular}{@{}ll@{}}
        \toprule
        & \textbf{Random init.} \hspace{4,2cm} \textbf{Shift-K init.} \\ \midrule
        Batch size & [20, 50, 100] \hspace{4,4cm} [20, 50, 100] \\
        Number Samples & 130000 \hspace{5,2cm} 130000\\
        Sequence length & 1500 \hspace{5,6cm} 1500 \\
        Position of $t^*$ & 200 \hspace{5,8cm} 200 \\
        Hidden neurons & 128 \hspace{5,8cm} 128 \\
        Input / output dimension & 1 \hspace{6,2cm} 1 \\
        Learning rates & [0.01, 0.005, 0.001, 0.0001] \hspace{2,2cm}[0.01, 0.005, 0.001, 0.0001] \\
        Weight decay & $10^{-5}$ \hspace{5,7cm}$10^{-5}$ \\ 
        $\rho$ & \{0, 0.2, 0.4, 0.6, 0.8, 1\} \hspace{2,7cm} \{0, 0.2, 0.4, 0.6, 0.8, 1\} \\ \midrule
        $a_s$ param. & $a_u = e^{-\alpha/K_{\textnormal{init}}}e^{i\epsilon_u\pi}, \varepsilon \sim \mathcal{U}(-1,1)$ \hspace{1cm} $a_u = e^{-\alpha/K_{\textnormal{init}}}e^{iu\frac{\pi}{K_{\textnormal{init}}}}$\\
         $b_s$ param. & $b_u = \frac{e^{-\alpha}(e^{2\alpha}-e^{-2\alpha})}{2K_{\textnormal{init}}}\times(-1)^u$ \hspace{1,8cm} $b_u = \frac{e^{-\alpha}(e^{2\alpha}-e^{-2\alpha})}{2K_{\textnormal{init}}}\times(-1)^u$  \\
        $\alpha$ & 1 \hspace{6,2cm} 1 \\
       $K_{\textnormal{init}}$ & 1300 \hspace{5,6cm} 1300 \\ \midrule
        Number epochs & 60 \hspace{6cm} 60 \\
        \bottomrule
    \end{tabular}
    \caption{{Experimental details for Figure~\ref{fig:xps} (left)}. We use $[\dots]$ to denote hyperparameters that were scanned over with grid search and $\{\dots\}$ to denote the variable of interest for the figure. We chose the same representation for $b_s$ in both cases because we observed small impact of this parameter on the final results.}
    \label{tab:xp-compare_init}
\end{table}

\begin{figure}
    \centering
    \includegraphics[width=1\linewidth]{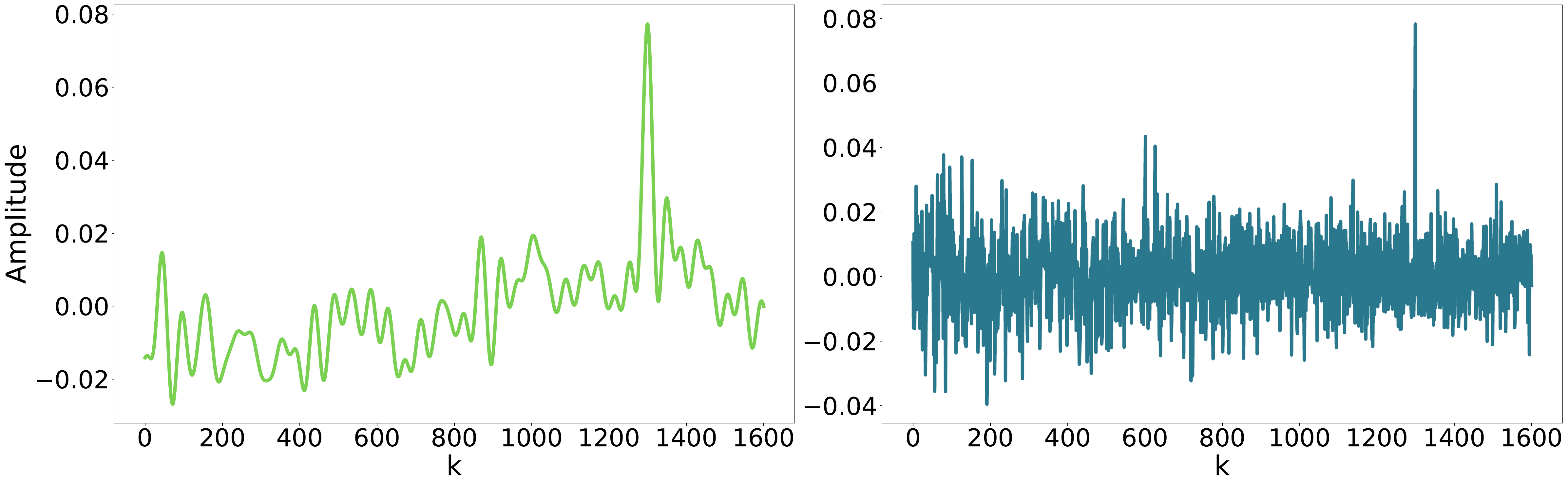}

    \vspace*{-.2cm}
    
    \caption{\textit{Comparison of Filters Obtained with Different Initialization Methods. Left: Filter obtained using our proposed shift-$K$ initialization method, which exhibits a more structured and interpretable pattern. Right: Filter obtained with random initialization, which appears significantly noisier, indicating less effective memory propagation. }}
    \label{fig:appendix 2 filters}
\end{figure}

\subsection{Robustness of Shift-K initialization}

In this second set of experiments, we investigate the robustness of our initialization scheme with respect to inaccuracies in the choice of $K_{\textnormal{init}}$ when initializing $a_s$ as in Eq.~\eqref{Param_of_the_as}. In practical applications, the actual shift of the sequence is often unknown, making it impossible to initialize with the exact optimal value of~$K$. A robust initialization method should exhibit resilience to such uncertainties, allowing for performance stability within a reasonable range of $K_{\textnormal{init}}$ values.

\begin{table}[ht]
    \centering
    \begin{tabular}{@{}l@{}}
        \toprule
        \textbf{Shift-K init.} \\ \midrule
        Batch size: [20, 50, 100] \\
        Number of Samples: 150000 \\
        Sequence length: 2250 \\
        Position of $t^*$: 250 \\
        Hidden neurons: 128 \\
        Input / output dimension: 1 \\
        Learning rates: [0.01, 0.005, 0.001, 0.0001] \\
        Weight decay: $10^{-5}$ \\ 
        $\rho$: 0.7 \\ \midrule
        $a_s$ param.: $a_u = e^{-\alpha/K_{\textnormal{init}}}e^{iu\frac{\pi}{K_{\textnormal{init}}}}$\\
        $b_s$ param.: $b_u = \frac{e^{-\alpha}(e^{2\alpha}-e^{-2\alpha})}{2K_{\textnormal{init}}}\times(-1)^u$  \\
        $\alpha$: 1 \\
        $K_{\textnormal{init}}$: \{250, 500, 1000, 2000, 4000, 8000, 16000, 32000\} \\ \midrule
        Number of epochs: 60 \\
        \bottomrule
    \end{tabular}
    \caption{{Experimental details for Figure~\ref{fig:xps} (right).}}
    \label{tab:xp-robustness}
\end{table}

\end{appendices}

\end{document}